\documentclass{article}

\usepackage{microtype}
\usepackage{graphicx}
\usepackage{subcaption}
\usepackage{booktabs} 

\usepackage{hyperref}

\usepackage[preprint]{icml2026}

\usepackage{graphicx}

\usepackage{multirow}
\usepackage{makecell}
\usepackage{algorithm}
\usepackage{algorithmic}

\usepackage{amsmath}
\usepackage{amssymb}
\usepackage{mathtools}
\usepackage{amsthm}

\usepackage[capitalize,noabbrev]{cleveref}

\theoremstyle{plain}
\newtheorem{theorem}{Theorem}[section]
\newtheorem{proposition}[theorem]{Proposition}
\newtheorem{lemma}[theorem]{Lemma}
\newtheorem{corollary}[theorem]{Corollary}
\theoremstyle{definition}

\newtheorem{assumption}[theorem]{Assumption}
\theoremstyle{remark}

\usepackage[textsize=tiny]{todonotes}

\icmltitlerunning{UniAttn: Reducing Inference Costs via Softmax Unification for Post-Training LLMs}

\begin{document}

\twocolumn[
  \icmltitle{UniAttn: Reducing Inference Costs via Softmax Unification \\ for Post-Training LLMs}



  \icmlsetsymbol{equal}{*}

  \begin{icmlauthorlist}
    \icmlauthor{Yizhe Xiong}{equal,thu}
    \icmlauthor{Wei Huang}{equal,bupt}
    \icmlauthor{Xin Ye}{equal,kuai}
    \icmlauthor{Hui Chen}{thu}
    \icmlauthor{Zijia Lin}{thu}
    \icmlauthor{Haoran Lian}{bhu}
    \icmlauthor{Zhenpeng Su}{kuai}
    \icmlauthor{Jungong Han}{thu_auto}
    \icmlauthor{Guiguang Ding}{thu}
  \end{icmlauthorlist}

  \icmlaffiliation{thu}{School of Software, Tsinghua University}
  \icmlaffiliation{bupt}{School of Computer Science, Beijing University of Posts and Telecommunications}
  \icmlaffiliation{kuai}{Kuaishou Technology}
  \icmlaffiliation{bhu}{Beihang University}
  \icmlaffiliation{thu_auto}{Department of Automation, Tsinghua University}

  \icmlcorrespondingauthor{Yizhe Xiong}{xiongyizhe2001@gmail.com}

  \icmlkeywords{UniAttn, Large Language Models, Softmax}

  \vskip 0.3in
]



\printAffiliationsAndNotice{}  

\begin{abstract}
Post-training is essential for adapting Large Language Models (LLMs) to real-world applications.
Deploying post-trained models faces significant challenges due to substantial memory overhead and noticeable inference latency.
Existing work has identified significant redundancies in LLMs and proposed efficient architectures, namely intra-layer KV sharing and cross-layer KV sharing.
However, these methods still result in high inference time overhead, remaining suboptimal for post-training pre-trained LLMs.
In this paper, we identify that the \texttt{Softmax} operation is a primary bottleneck for LLM inference and discover that it is actually highly redundant during post-training. 
We propose Softmax \textbf{Uni}fication in \textbf{Att}e\textbf{n}tion (\textbf{UniAttn}), a novel post-training method that unifies Softmax activations across transformer blocks to reduce LLM inference costs. Additionally, UniAttn adopts a linear projection to compensate for the errors induced by Softmax unification.
Experiments show that UniAttn matches the performance of standard post-training while significantly reducing inference costs, outperforming existing efficient architectures during post-training.
\end{abstract}

\section{Introduction}
\label{sec:introduction}

Post-training \cite{TULU3} has become a critical step in developing advanced LLMs for both general-purpose tasks \cite{GPT4, LLaVA} and domain-specific tasks \cite{thirunavukarasu2023large, PMC, DeepSeekMath}. 
It typically refers to performing supervised fine-tuning (SFT) or domain-specific continual pre-training on pre-trained LLMs using a text corpus different from the original pre-training dataset.
Current post-training approaches for real-world applications involve fine-tuning popular base models \cite{LLaMA3, Mistral, Gemma2} with standard decoder-based architectures that have scaled to hundreds of billions of parameters. 
However, these approaches face two widely recognized challenges in deploying post-trained models \cite{StreamingLLM, PrefixKV, LLMDrop, CLA, FastV}: (1) the substantial memory overhead required for KV-cache storage during inference, and (2) the computational latency, which leads to noticeable inference time.

To address those challenges, recent research has proposed KV-cache sharing methods, and showed that they offer notable advantages over traditional pruning \cite{LLMPruner,H2O} and quantization \cite{SmoothQuant,AWQ} methods: While pruning maintains model accuracy with only modest efficiency gains, and quantization reduces inference overhead at the expense of substantial performance degradation, KV-cache sharing methods achieve efficiency improvements with minimal performance drops.  
To this end, researchers have identified significant redundancies \textit{inside and across LLM layers}, leading to two mainstream KV-cache sharing paradigms: intra-layer KV sharing, including MQA \cite{MQA} and GQA \cite{GQA}, and a more aggressive cross-layer KV sharing, including CLA \cite{CLA} and YOCO \cite{YOCO}.
Despite achieving impressive KV-cache compression ratios, cross-layer KV sharing yields limited inference time reduction: less than 5\% at deployment (see \cref{sec:exp} for detailed results).
This raises a critical question: \textit{How to design LLM architectures that achieve both inference time and memory efficiency?}

To further explore that question, we delve deeper into LLM architectures and make two key observations. First, the \textit{Softmax operation} significantly contributes to inference costs. Performing the Softmax operation requires access to the entire K-cache. Although the Softmax operation accounts for less than 1\% of the total FLOPs compared to the linear projections in the backbone, it results in higher latency than the linear projections due to limited parallelism \cite{SimA,Consmax}.
Second, \textit{Softmax activations} exhibit high cross-layer similarity in the \textit{top half layers} across various open-source pre-trained LLMs and post-training datasets. 
For instance, the average cross-layer similarity exceeds 0.9 on the top half layers in LLaMA-3.1 8B \cite{LLaMA3} when evaluated on the Tulu3 \cite{TULU3} training set (see \cref{fig:softmax_sim} for detailed results).
This observation highlights the potential for an effective and generalizable approach to leveraging Softmax redundancies during post-training.

In this paper, we propose to post-train LLMs with Softmax \textbf{Uni}fication in \textbf{Att}e\textbf{n}tion (\textbf{UniAttn}). Specifically, we group consecutive transformer blocks in the LLM as several ``Superblocks'', and unify the Softmax activations of all blocks within each Superblock. This approach significantly reduces memory and time costs. 
Additionally, we demonstrate that the error introduced by unifying Softmax activations can be effectively compensated by a linear projection, for which we design an initialization and training pipeline.
To better interpret the performance gains of UniAttn, we present a theoretical analysis comparing it with existing cross-layer KV sharing methods.
Our analysis reveals that UniAttn better preserves the benefits associated with model depth \cite{DeepNet}, enabling the model to acquire new capabilities during post-training.

\begin{figure}
\vspace{-0.05in}
\begin{center}
\centerline{\includegraphics[width=\columnwidth]{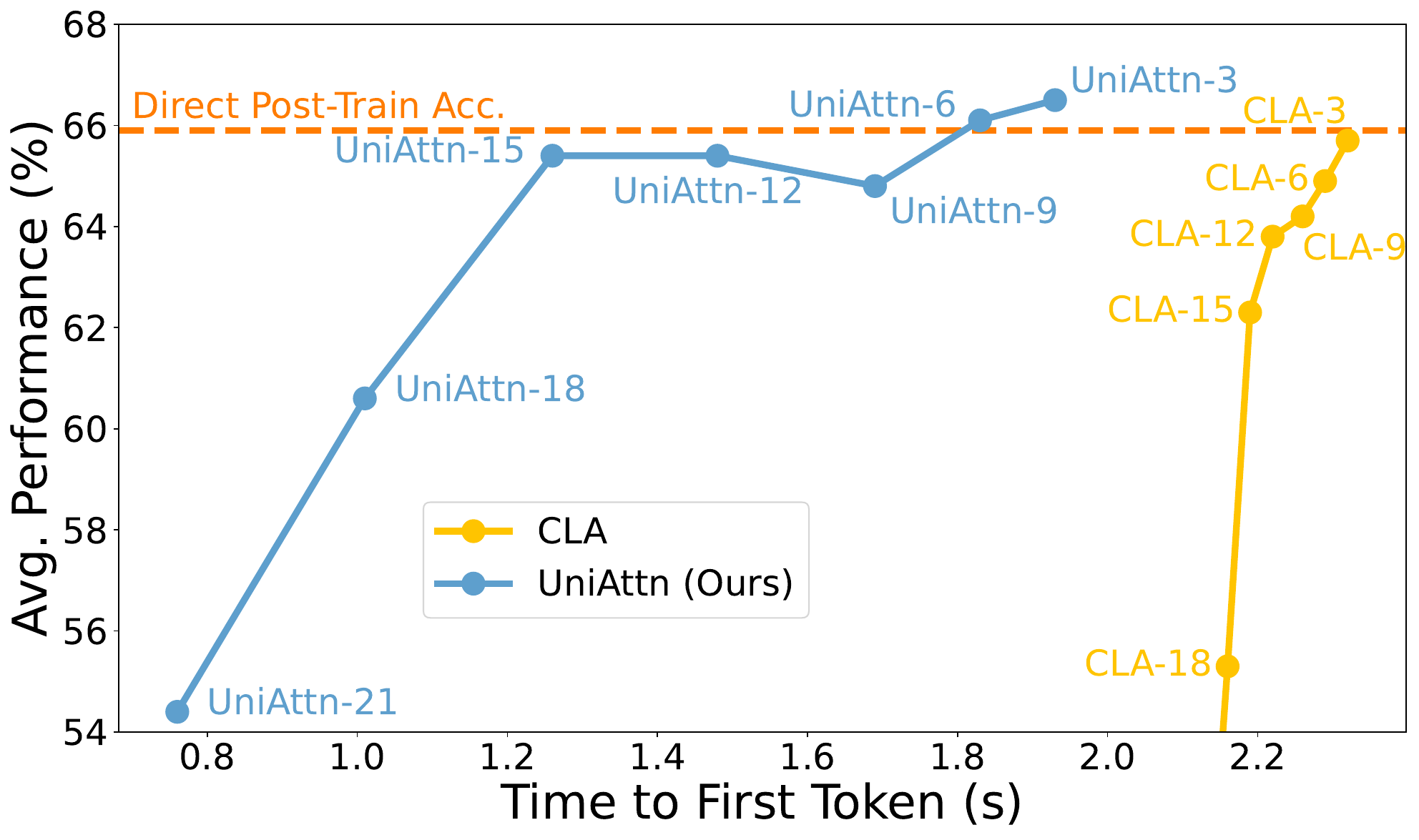}}
\caption{Comparisons of our UniAttn and directly applying cross-layer KV sharing (CLA) during post-training. ``A-X'' represents modifying total of X layers when applying A. }
\label{fig:teaser}
\end{center}
\vspace{-0.35in}
\end{figure}

We conduct extensive experiments on four open-source pre-trained LLMs across two post-training scenarios: enhancing domain-specific capabilities and improving general capabilities. Our results show that UniAttn consistently achieves performance comparable to standard post-training while reducing inference costs. Compared to existing KV-cache sharing baselines, UniAttn achieves better performance with substantially lower inference costs (see \cref{fig:teaser} for an overview). Furthermore, our UniAttn can be combined with intra-layer KV sharing and KV-cache compression methods to further cut down the memory overhead, showing strong practicality for real-world applications.

Overall, we summarize our contribution as follows:

\begin{itemize}
    \item We identify that the Softmax operations in pre-trained LLMs, although yielding high inference costs, exhibit significant redundancies as Softmax activations share high cross-layer similarities. 
    \item Based on the observations, we propose Softmax \textbf{Uni}fication in \textbf{Att}e\textbf{n}tion (\textbf{UniAttn}) for post-training LLMs. Specifically, we unify the Softmax activations across decoder blocks. Additionally, we leverage linear projection to compensate for the error from Softmax unification. We show theoretically that UniAttn better preserves model capabilities than compared baselines.
    \item Extensive experiments show that UniAttn achieves comparable performance to standard post-training while reducing inference costs, outperforming existing efficient architectures. 
\end{itemize}

\section{Related Works}

\subsection{Post-Training LLMs}

Building frontier LLMs for real-world applications involves two crucial stages: pre-training and post-training. Since pre-training data and methodologies are often proprietary, the research community has extensively explored post-training upon open-source pre-trained LLMs \cite{LLaMA3,Gemma2,Mistral} to enhance their \textit{general} or \textit{domain-specific} capabilities for deployment \cite{lian2024breakingstagebarriernovel}.
Post-training is typically performed on instruction-following or domain-specific datasets.
Recently, various datasets have been curated to equip LLMs with specific abilities, including general instruction-following models \cite{Hermes, TULU3, InfinityInstruct2024}, medical QA models \cite{PMC, Aloe, singhal2023expertlevelmedicalquestionanswering}, legal QA models \cite{LawyerLLaMA, LawGPT}, and models with strong mathematical problem-solving capabilities \cite{Goat}.

Existing research on post-training mainly focuses on creating specific datasets for equipping open-source LLMs with specific capabilities. 
Differently, we investigate the redundancies in pre-trained LLMs and leverage them for post-training inference-efficient LLMs.

\subsection{Efficient KV Sharing Architectures} 
Unlike existing pruning and quantization approaches \cite{ToMe,PYRA}, efficient KV sharing architectures utilize structural redundancies to create inference-efficient model variants for deployment.
Those architectures mainly fall into two categories: intra-layer KV sharing, including MQA \cite{MQA} and GQA \cite{GQA}, and cross-layer KV sharing, including CLA \cite{CLA} and YOCO \cite{YOCO}.
Specifically, MQA \cite{MQA} simplifies the attention mechanism by utilizing multiple query heads and a single KV head. 
GQA \cite{GQA} takes a step further from MQA by organizing query heads as multiple groups and assigns different KV heads to each group.
CLA \cite{CLA} proposes a cross-layer sharing mechanism to further reduce KV-cache memory overhead by sharing KV-cache across different layers. 
YOCO \cite{YOCO} transforms the original structure into self-decoders and cross-decoders, and adopts a global KV-cache across decoders.

Existing intra-layer KV sharing works have been adopted by various open-source LLMs \cite{LLaMA3,Mistral,Gemma2}, yielding modest reduction on inference memory overhead. Although cross-layer KV sharing introduces more aggressive memory reduction, according to our analysis in \cref{sec:introduction} and \cref{sec:analysis_method}, it still suffers from efficiency and performance issues. To address this, we propose UniAttn, which achieves promising performance for post-training LLMs and outperforms competing methods.

\section{Methodology}

\subsection{Preliminaries}

In this paper, we focus on mainstream decoder-only LLMs with a transformer structure. Each transformer block in an LLM features a multi-head self-attention (MHA) module and a feed-forward network (FFN). Given an input sequence $\mathbf{x}\in \mathbb{R}^{l\times d}$ where $l$ denotes the sequence length and $d$ denotes the hidden size, both MHA and FFN generate an output sequence with identical length and shape.
We focus on the MHA module. Formally, we denote MHA output in layer $i$ as $\mathbf{x}_i'$, where $\mathbf{x}_i'=\text{MHA}(\text{Norm}(\mathbf{x}_i))+\mathbf{x}_i$. ``$\text{Norm}(\cdot)$'' denotes the Pre-Norm, a component adopted by mainstream open-source LLMs \cite{LLaMA3,Mistral,Gemma2,bai2023qwen}.
In the MHA module in layer $i$, each token in the input sequence $\mathbf{x}_i$ is first projected by $W_{Q,i}$, $W_{K,i}$, and $W_{V,i}$, forming $Q_i,K_i\in d\times d_k$ and $V_i\in d\times d_v$. Then, the Softmax activation $s_i$ is calculated by: 
\begin{equation}
\label{eq:softmax_activation}
    s_i=\text{softmax}(\frac{Q_iK_i^T}{\sqrt{d_k}}).
\end{equation}
Subsequently, $s_i$ is projected using $V_i$ and the output weight matrix $W_{o,i}$ of the $i$th layer, and the input to MHA is added back through a residual connection to produce the MHA output $\mathbf{x}_i'$: 
\begin{equation}
    \mathbf{x}_i'=s_iV_iW_{o,i}+\mathbf{x}_i
\end{equation}

\subsection{UniAttn: Softmax Unification in Attention}
\label{sec:uniattn}

Although existing KV-cache sharing methods--such as intra-layer and cross-layer KV sharing--can reduce the memory footprint, they offer limited inference time reduction.
To achieve both memory-efficiency and inference acceleration, we investigate the Softmax operation because (1) calculating Softmax demands the entire K-cache that accounts for 50\% memory of the KV-cache; and (2) multiple studies \cite{DBLP:conf/accv/GengLZKAC18,SimA} have recognized that the Softmax operation leads to high inference latency. We study the redundancy of Softmax activations in pre-trained LLMs by measuring the importance of Softmax operations in each layer. To ensure the generalizability of our study, 
we employ post-training datasets PMC \cite{PMC}, Infinity Instruct (Math) \cite{InfinityInstruct2024}, Tulu3 \cite{TULU3}, and long samples (F15K) provided by \cite{rerope2023}, and utilize LLaMA-2 7B \cite{touvron2023llama}, LLaMA-3.1 8B \cite{LLaMA3}, Mistral 7B \cite{Mistral}, and Gemma-2 9B \cite{Gemma2}. 
We collect all Softmax activations, i.e., the $s_i$ vectors in \cref{eq:softmax_activation}, and compute the cosine similarity $\text{Sim}(i)=\cos(s_i,s_{i-1})$ for each sequence. We then average $\text{Sim}(i)$ over all sequences. 

\underline{\textit{Observation:}} As shown in \cref{fig:softmax_sim}, Softmax activations in \textit{top half layers} share a high cross-layer similarity across different settings, showing that Softmax operations are highly redundant in LLMs.
See \cref{sec:append_exp} for quantitative analysis. 

\begin{figure}
\begin{center}
\centerline{\includegraphics[width=\columnwidth]{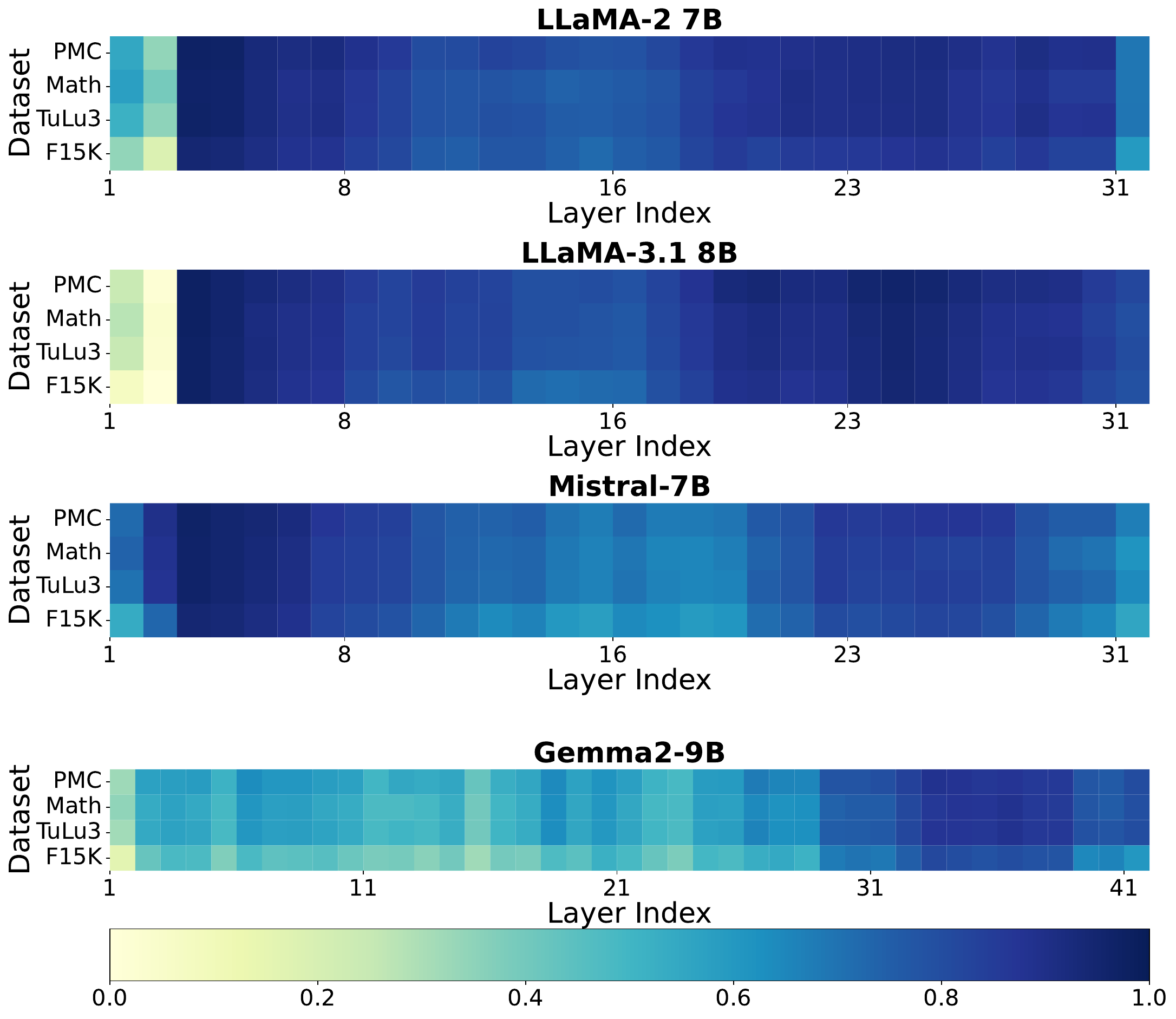}}
\caption{Cosine similarity results of average Softmax activations. Across all settings, the average Softmax activations of the top half of layers (i.e., the right half of each heatmap bar) share a high cosine similarity. }
\label{fig:softmax_sim}
\end{center}
\vspace{-0.25in}
\end{figure}

Motivated by such redundancy, we propose Softmax \textbf{Uni}fication in \textbf{Att}e\textbf{n}tion (\textbf{UniAttn}) for post-training LLMs. Specifically, UniAttn groups consecutive decoder blocks in LLMs as ``Superblocks'', and unifies the Softmax activations inside each Superblock to reduce the full Softmax operations and KV-cache size during inference. Additionally, we observe that the error introduced by Softmax unification can be compensated by a linear projection, and in turn design a pipeline to initialize and train the linear projection. 
UniAttn simultaneously reduces the GPU memory and inference latency while achieving competitive performance. 
Different from cross-layer KV sharing, UniAttn does not diminish the ``depth factor''. We show this in \cref{sec:discussion}.

\textbf{Unify Softmax Activations in Superblocks.}
To apply UniAttn, we first merge several consecutive transformer blocks as ``SuperBlocks''. For simplicity, we create SuperBlocks of the same size. Inside each SuperBlock, only the bottom block employs the Softmax operation. Other blocks simply re-use the Softmax activation matrix $s$ calculated by the bottom block in the SuperBlock. Formally, for layer $i+b$ that re-use the Softmax activation $s_i$ from layer $i$:
\begin{equation}
    \mathbf{x}_{i+b}'=s_iV_{i+b}W_{o,i+b}+\mathbf{x}_{i+b}
\end{equation} 
Since only the bottom block contribute to calculating the Key projection, other blocks that reuse Softmax activations do not store their corresponding K-cache, leading to GPU memory-efficiency. Moreover, removing Softmax calculations in those layers also contribute to inference latency reduction.
Similar to cross-layer KV sharing, UniAttn is also orthogonal to intra-layer KV sharing. In \cref{sec:exp}, we employ pre-trained LLMs with GQA to evaluate UniAttn.

\textbf{Linear Compensation.}
Simply unifying Softmax activations unavoidably introduces errors in the forward pass. Formally, the error $\epsilon$ in layer $i+b$ is:
\begin{equation}
\begin{aligned}
    \epsilon&=\mathbf{x}_{i+b}^{\text{ori}\ \prime}-\mathbf{x}_{i+b}^{\text{uni}\ \prime} = (\mathbf{x}_{i+b}^{\text{ori}}-\mathbf{x}_{i+b}^{\text{uni}})\\
    &+(\text{softmax}(\frac{Q_{i+b}K^T_{i+b}}{\sqrt{d_k}})-s_i)V_{i+b}W_{o,i+b}
\end{aligned}
\end{equation}
where $\mathbf{x}_{i+b}^{\text{ori}}$ and $\mathbf{x}_{i+b}^{\text{Uni}}$ denote the input to layer $i+b$ in the original model and the model with UniAttn applied, respectively. According to \cite{SimA}, calculating linear projections yield \textit{significantly lower} latency compared to the Softmax operation. And thus we study the effectiveness of compensating $\epsilon$ with a linear projection matrix $W_{c}$. To compensate for $\epsilon$, we apply $W_{c}$ to each layer with unified Softmax activations as:
\begin{equation}
    \mathbf{x}_{i+b}'=s_iV_{i+b}W_{o,i+b}+\mathbf{x}_{i+b}+\mathbf{x}_{i+b}W_{c,i+b}
\end{equation}
To find a proper initialization for $W_c$ in a training-free manner, we propose to directly compensate $\epsilon$:
\begin{theorem}
\label{theorem:init}
    The initialization for $W_{c}$ that incurs minimal error when compensating $\mathbb{E}(\epsilon)$ satisfies:
    \begin{equation}
        W_{c}=V\Sigma^{+}U^T\mathbb{E}(\epsilon),
    \end{equation}
    where $U\Sigma V^T$ denotes the SVD decomposition of $\mathbb{E}(\mathbf{x}_{i+b})$, $\Sigma^{+}$ denotes the pseudoinverse of $\Sigma$.
\end{theorem}
We leave the proof in the Appendix. 
In practice, we forward a small portion of training data from the post-training datasets, and use the average values $\bar{\mathbf{x}}_{i+b}$ and $\bar{\epsilon}$ as estimations of $\mathbb{E}(\mathbf{x}_{i+b})$ and $\mathbb{E}(\epsilon)$. Considering that compensating for errors in previous layers would influence the calculation in subsequent layers, we calculate the initialization of $W_c$ bottom-up.  
Experiments on the F15K dataset \cite{rerope2023} with sequence length 5,120 show that, after linear compensation, the errors for LLaMA-2 7B and LLaMA-3.1 8B are reduced to 2.9\% (11.09/382.87) and 6.4\% (5.76/89.58) of the original errors without compensation, respectively.
This clearly demonstrates the effectiveness of the linear compensation.
Please refer to \cref{sec:more_linear_comp} for more theoretical insights on the linear compensation.

To better train the linear transformations, we adopt a two-stage post-train pipeline. In the first stage, we only fine-tune the $W_c$ weights \textit{with early stop} and keep other weights frozen. 
We then conduct a full fine-tuning in the second stage. We formalize a detailed pipeline in \cref{alg:example} in the Appendix. 

\begin{figure*}[t]
\vskip 0.2in
\begin{center}
\centerline{\includegraphics[width=0.99\linewidth]{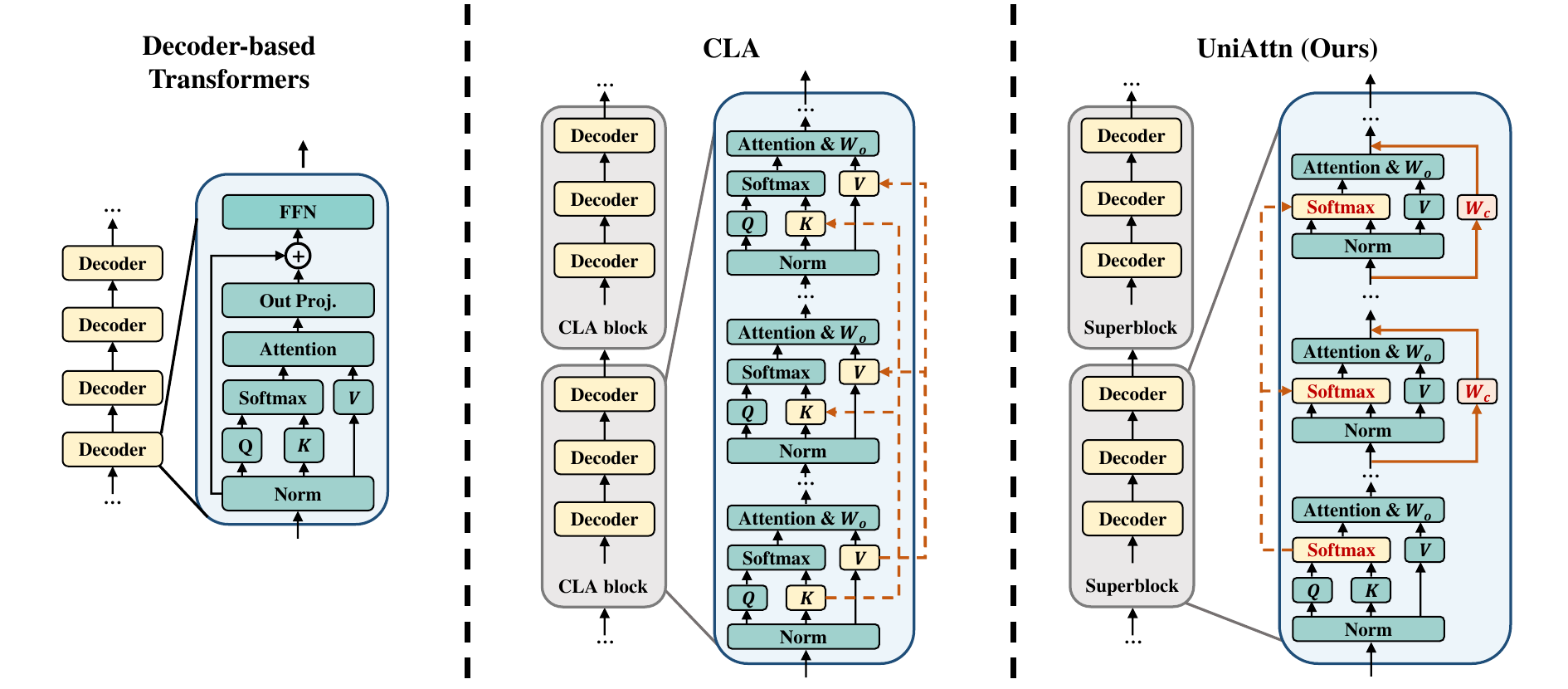}}
\caption{Pipeline comparison between standard decoder-based transformers, CLA \cite{CLA} (block size 3), and UniAttn (Superblock size 3). UniAttn shares the Softmax activations across layers in grouped Superblocks and adds linear transformation $W_c$ to compensate for the unification error. For simplicity, only the self-attention calculation is illustrated for CLA and UniAttn.}
\label{fig:pipeline}
\end{center}
\vskip -0.2in
\end{figure*} 

\subsection{Theoretical Analysis on Existing Methods} 
\label{sec:analysis_method}

Existing KV-cache sharing methods can be mainly categorized as intra-layer KV sharing and cross-layer KV sharing methods. Intra-layer KV sharing methods can be directly combined with UniAttn (see \cref{sec:exp} for detailed results). 
While cross-layer KV sharing methods achieve a more aggressive memory reduction rate, through theoretical analysis, we show that it diminishes the impact of model depth on model activations through directly sharing the same KV-cache across layers. This negatively affects the capabilities of pre-trained models, as model depth is a critical factor in determining their overall performance \cite{DeepNet}.

Existing research \cite{razzhigaev-etal-2024-transformer} has shown that LLM layers can be approximated as linear transformations. Together with the conslusion on training gradient dynamics \cite{DBLP:conf/icml/XiongYHZZXZLWL20},
we first give a reasonable assumption:
\begin{assumption}
\label{assum:inftesimal}
In top layers of \textbf{pre-trained} Pre-Norm LLMs:
\begin{equation}
    \mathbf{x}_{i+1}=\mathbf{x}_{i}+\delta_i,\quad \text{where}\quad||\delta_i||\ll ||\mathbf{x}_{i}||.
\end{equation}
\end{assumption}
We define $\delta$ as the ``depth factor'' that model applies on activations for further discussion. $\delta$ represents the extent to which model depth influences activations. \emph{Please refer to \cref{append:theo_insights} for theoretical insights on proposing this assumption.}

\textbf{Analysis on Cross-layer KV Sharing:} We adopt \cref{assum:inftesimal} to analyze the cross-layer KV sharing architecture, in which KV-cache is shared across multiple consecutive layers. Suppose layer $i+1$ shares the KV-cache from layer $i$, the MHA operation in layer $i+1$ can be written as:
\begin{equation}
\label{eq:cla}
    \mathbf{x}'_{i+1}=\text{softmax}(\frac{QK^T_{i}}{\sqrt{d_k}})V_iW_{o,i+1}+\mathbf{x}_{i+1},
\end{equation}
where $K_i$ and $V_i$ are the $K$ and $V$ matrices from layer $i$. We propose that:
\begin{proposition}
\label{prop:effective_depth}
    In \textbf{pre-trained} Pre-Norm LLMs, cross-layer KV sharing diminishes the ``depth factor'' $\delta$ on activations.
\end{proposition}

Please refer to \cref{sec:effective_depth_proof} for detailed proof.
\Cref{prop:effective_depth} gives a rigorous proof showing that cross-layer KV sharing diminishes the impact of model depth on model activations. To offer an intuitive understanding of how KV sharing alters the structure of QKV computation within self-attention, we reformulate the operation as follows:
\begin{equation}
\label{eq:cla_equiv}
\begin{aligned}
    \text{QKV-Comp}_{i+1}
    &=\text{softmax}(A_i)V_i\\
    &=\text{softmax}(\frac{\mathbf{x}_{i}W_{q,i+1}K^T_{i}}{\sqrt{d_k}})V_i.
\end{aligned}
\end{equation}
Compared to QKV-comp$_i$ in the previous layer, the only difference lies in the use of distinct $W_q$ matrices, suggesting that the QKV computation in layer $i+1$ can be viewed as introducing an additional query head applied to the same KV pairs as in layer $i$. In this sense, \textit{cross-layer KV sharing implicitly transforms QKV computation in different layers into additional query heads}.

\subsection{Discussion on UniAttn}
\label{sec:discussion}
Our UniAttn reduces both memory cost and inference latency simultaneously. 
We underscore that unlike cross-layer KV sharing, our UniAttn does not diminish the ``depth factor'' $\delta$ in the forward pass. Please refer to \cref{sec:uniattn_depth_proof} for detailed proof.
\begin{proposition}
\label{prop:uniattn_depth}
    In pre-trained Pre-Norm LLMs, the ``depth factor'' in UniAttn is unignorable. 
\end{proposition}

\begin{table*}[t]
    \centering
    \resizebox{0.9\textwidth}{!}{\begin{tabular}{c|cc|ccc|c|cc|c}
\toprule
\multirow{2}{*}{Method} & \multirow{2}{*}{\makecell{TTFT \\ (s)}} & \multirow{2}{*}{\makecell{KV \\ Cache}} & \multicolumn{4}{c|}{Medical} & \multicolumn{3}{c}{General}    \\ 
&  & & PubMedQa   & MedMCQA  & \multicolumn{1}{c}{MedQA} & AVG &  SIQA  & \multicolumn{1}{c}{CommonsenseQA} & AVG  \\ \midrule
\multicolumn{10}{c}{\textbf{LLama3.1-8B (w/ GQA)}}  \\ 
\midrule
Pre-Trained & 2.33 & 100\%  & 75.0   & 56.8  & 59.1   & 63.6 & 46.8 & 71.7  & 59.3 \\
Post-Train & 2.33 & 100\%  & 75.4   & 59.0  & 63.2   & 65.9 & 50.1 & 73.9  & 62.0 \\ \hline
LLMDrop-Half & 1.77 & 81.3\%  & 75.6   & \underline{57.8}  & \underline{61.5}   & \underline{65.0} & 51.4 & \underline{75.0}  & \underline{63.2}  \\ 
LLMDrop-Full & \textbf{1.47} & 62.5\%  & \underline{78.2}   & 53.6  & 55.9   & 62.6 & \textbf{51.8} & 73.3  & 62.6 \\ 
CLA-Half & 2.29 & 81.3\%  & 73.8   & \textbf{58.1}  & \textbf{62.8}   & 64.9 & \underline{51.7} & 74.7  & \underline{63.2} \\ 
CLA-Full & 2.22 & 62.5\%  & 75.4   & 56.9  & 59.2   & 63.8 & 50.7 & 71.2  & 61.0 \\ 
UniAttn (Ours) & \underline{1.48} & 81.3\%  & \textbf{79.0}   & 57.6  & 59.5   & \textbf{65.4} & 51.1 & \textbf{75.6}  & \textbf{63.4} \\ 
\midrule
\multicolumn{10}{c}{\textbf{LLama2-7B (w/o GQA)}}  \\ 
\midrule
Pre-Trained & 2.20 & 100\%  & 71.4   & 32.2  & 34.7   & 46.1 & 50.2 & 51.1  & 52.7 \\
Post-Train & 2.20 & 100\%  & 75.4   & 48.8  & 49.2   & 57.8 & 51.6 & 66.3  & 59.0 \\ \hline
LLMDrop-Half & 1.74 & 81.3\%  & \underline{75.2}   & 48.4  & 49.8   & \underline{57.8} & \underline{51.7} & \underline{66.7} & \underline{59.2}  \\ 
LLMDrop-Full & \textbf{1.37} & 62.5\%  & 75.0   & \underline{48.8}  & 48.8 & 57.5 & 51.2 & 65.2  & 58.2 \\ 
CLA-Half & 2.20 & 81.3\%  & 74.4   & 48.4  & \textbf{50.6}   & \underline{57.8} & 51.3 & 66.1  & 58.7 \\ 
CLA-Full & 2.23 & 62.5\%  & \underline{75.2}   & 45.7  & 47.5   & 56.1 & 50.2 & 65.6  & 57.9 \\ 
UniAttn (Ours) & \underline{1.44} & 81.3\%  & \textbf{75.6}   & \textbf{49.4}  & \underline{50.1}   & \textbf{58.4} & \textbf{51.9} & \textbf{67.7}  & \textbf{59.8} \\ 
\bottomrule
 \end{tabular}}
 \caption{Post-training performance (\%) comparison. \textbf{Bold} and \underline{underline} denote the best and second-best performance of compressed models. We also report time to first token (TTFT) and KV-cache retain rate (KV Cache). }
 \label{tab:main}
\end{table*}

\section{Experiments}
\label{sec:exp}

\subsection{Experimental Settings}
\textbf{Post-Training Settings.} 
We consider two post-training settings to validate our UniAttn: fine-tuning on domain-specific datasets and on general instruction following datasets. 
We choose PMC \cite{PMC} as the domain-specific dataset and evaluate fine-tuned models on PubMedQA \cite{PubMedQA}, MedMCQA \cite{MedMCQA}, and MedQA \cite{MedQA}.
We choose Tulu3 \cite{TULU3} as the general instruction following dataset and evaluate fine-tuned models on SIQA \cite{SIQA} and CommonsenseQA \cite{CommonsenseQA}. 
We employ \cite{eval-harness} to perform all benchmark evaluations.

\textbf{Model.} 
We post-train 4 open-source pre-trained LLMs for evaluating our UniAttn: LLaMA-2 7B \cite{touvron2023llama}, LLaMA-3.1 8B \cite{LLaMA3}, Mistral 7B \cite{Mistral}, and Gemma-2 9B \cite{Gemma2}. We use base models that have undergone only pre-training.

\textbf{Implementation details.} 
We adopt the pattern in \cref{fig:softmax_sim} and apply UniAttn in the \textit{top half of layers}. Unless otherwise noted, we adopt Superblock size $=4$, which yields a total of 4 Superblocks in LLaMA-2 7B, LLaMA-3.1 8B, Mistral 7B, and 5 Superblocks in Gemma-2 9B. For all experiments, we post-train LLMs for 1 epoch \cite{TULU3}. We adopt all training hyperparameter values based on existing research. All training experiments were conducted on 8 H800 GPUs. For time to first token (TTFT) time measurement, we use a single H800 GPU and a context length of 8,192 and measure the time to generate the first token after receiving the sequence input. See the Appendix for more details. 

\subsection{Performance Comparison on Post-Training}

We verify the effectiveness of UniAttn on 2 open-source LLMs with different architectures, namely LLaMA-3.1 8B (with GQA) \cite{LLaMA3} and LLaMA-2 7B (without GQA) \cite{touvron2023llama}. 
We use LLMDrop \cite{LLMDrop} and CLA \cite{CLA} as competing efficient architectures, as both methods exploit the cross-layer redundancies in LLMs. For LLMDrop, we bypass the MHA modules with the highest input-output similarities in the top half of the layers. For CLA, we group the top half of the layers into CLA blocks, using the same configuration as in the grouping of Superblocks.
Since both competing methods drop the entire KV-cache in the operated layers, we compare two configurations: one with the same operating layers (X-Full) and another with the same KV-cache compression ratio (half operating layers, X-Half). We also conduct post-training after applying these methods.
Same conclusions are observed on Mistral 7B \cite{Mistral} and Gemma-2 9B \cite{Gemma2}. Detailed results are in the Appendix due to the page limit. 

\textbf{Main Conclusion.} We compare different efficient architectures against directly post-training the original LLM (Post-Train) and using the pre-trained model (Pre-Trained) as baselines, as shown in \cref{tab:main}. Among all competing architectures, our UniAttn achieves the best overall performance across model structures (with and without GQA) and post-training datasets. Besides significantly reducing inference costs in both time and memory, UniAttn maintains comparable performance to directly post-training the original model, well demonstrating its effectiveness.
When comparing UniAttn with LLMDrop-Full, we observe that Softmax unification provides a similar acceleration rate to bypassing the entire MHA module, but it achieves significantly better post-training performance. This aligns with our finding that Softmax operations are both costly and redundant. When comparing to CLA, UniAttn achieves better performance with significantly lower inference latency, demonstrating that the UniAttn architecture is better suited for post-training than cross-layer KV sharing. 

\textbf{Analysis of CLA.} 
Since CLA only reduces the latency associated with KV projections, its impact on TTFT is minimal. As a result, CLA-Half, CLA-Full, and the original model exhibit comparable TTFT times.
Furthermore, CLA achieves comparable performance to LLMDrop when applied to the same number of layers. This suggests that the cross-layer attention mechanism employed by CLA is structurally equivalent to directly bypassing the MHA module. This observation supports our theoretical analysis that cross-layer KV sharing methods like CLA can diminish the depth factor in pre-trained LLMs, limiting their performance. 

\subsection{Experimental Analysis}
We choose LLaMA-3.1 8B \cite{LLaMA3} and the medical domain to conduct experimental analysis. See \cref{sec:append_exp} for more analysis.

\begin{table*}[t]
    \centering
    \resizebox{0.9\textwidth}{!}{\begin{tabular}{c|ccc|cccc}
\toprule
Method & Unify Softmax & Compensation & Init $W_c$ & PubMed & MedMCQA & MedQA & AVG \\
\midrule
Post-Train & - & - & - & 75.4 & 59.0  & 63.2 & 65.9 \\ \hline 
\multirow{4}{*}{UniAttn} & $\surd$ & $\times$ & $\times$ & 77.0 & 55.7 & \underline{58.8} & 63.8 \\
 & $\surd$ & $\surd$ & zero-init & 78.0 & 55.8 & 58.1 & 63.9 \\
 & $\surd$ & $\surd$ & \cref{theorem:init} & \underline{78.2} & \underline{57.1} & 58.4 & \underline{64.6} \\
 & $\surd$ & $\surd$ & \cref{theorem:init} + fine-tune & \textbf{79.0} & \textbf{57.6} & \textbf{59.5} & \textbf{65.4} \\
\bottomrule
\end{tabular}}
\caption{Ablation study results. \textbf{Bold} and \underline{underline} denote the best and second-best UniAttn performance.}
\label{tab:ablation}
\end{table*}

\begin{table}[t]
    \centering
    \resizebox{\columnwidth}{!}{\begin{tabular}{c|ccc|c}
        \toprule
        Method & TTFT (s) & H$_2$O Comp. & KV-Cache & AVG \\
        \midrule
        Post-Train & 2.33 & - & 100\% & 65.9 \\
        CLA-Half & 2.29 & - & 81.3\% & 64.9 \\
        CLA-Full & 2.22 & - & 62.5\% & 63.8 \\ \midrule
        UniAttn & 1.48 & - & 81.3\% & 65.4 \\
        \multirow{3}{*}{UniAttn+H$_2$O} & 1.48 & 20\% & 65.0\% & 65.2 \\
        & 1.46 & 40\% & 48.8\% & 65.0 \\
        & 1.46 & 60\% & 32.5\% & 63.3 \\
        \bottomrule
    \end{tabular}}
    \caption{Results of inference costs analysis. ``H$_2$O Comp.'' denotes the compression rate introduced by H$_2$O and ``KV-Cache'' denotes the KV-Cache retain rate.}
    \label{tab:uniattn_h2o}
\end{table}

\begin{figure}[t]
    \centering
    \includegraphics[width=0.7\linewidth]{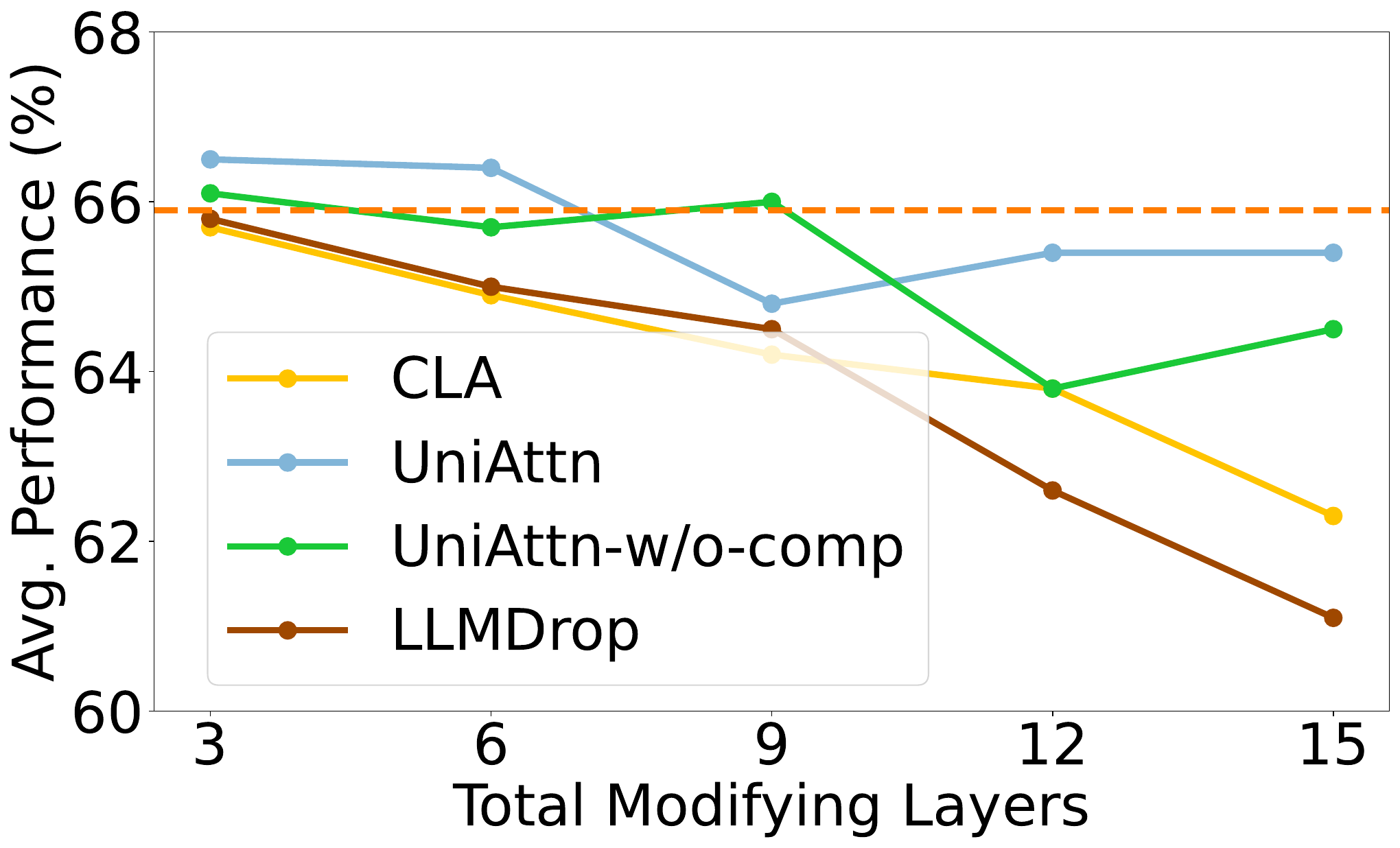}
    \caption{Comparison between baselines, UniAttn, and UniAttn without (``w/o'') linear compensation.}
    \label{fig:wo_proj}
\end{figure}


\begin{figure}[t]  
    \begin{minipage}{1.0\columnwidth}  
        \centering
        \begin{minipage}{0.49\textwidth}  
            \centering
            \includegraphics[width=\textwidth]{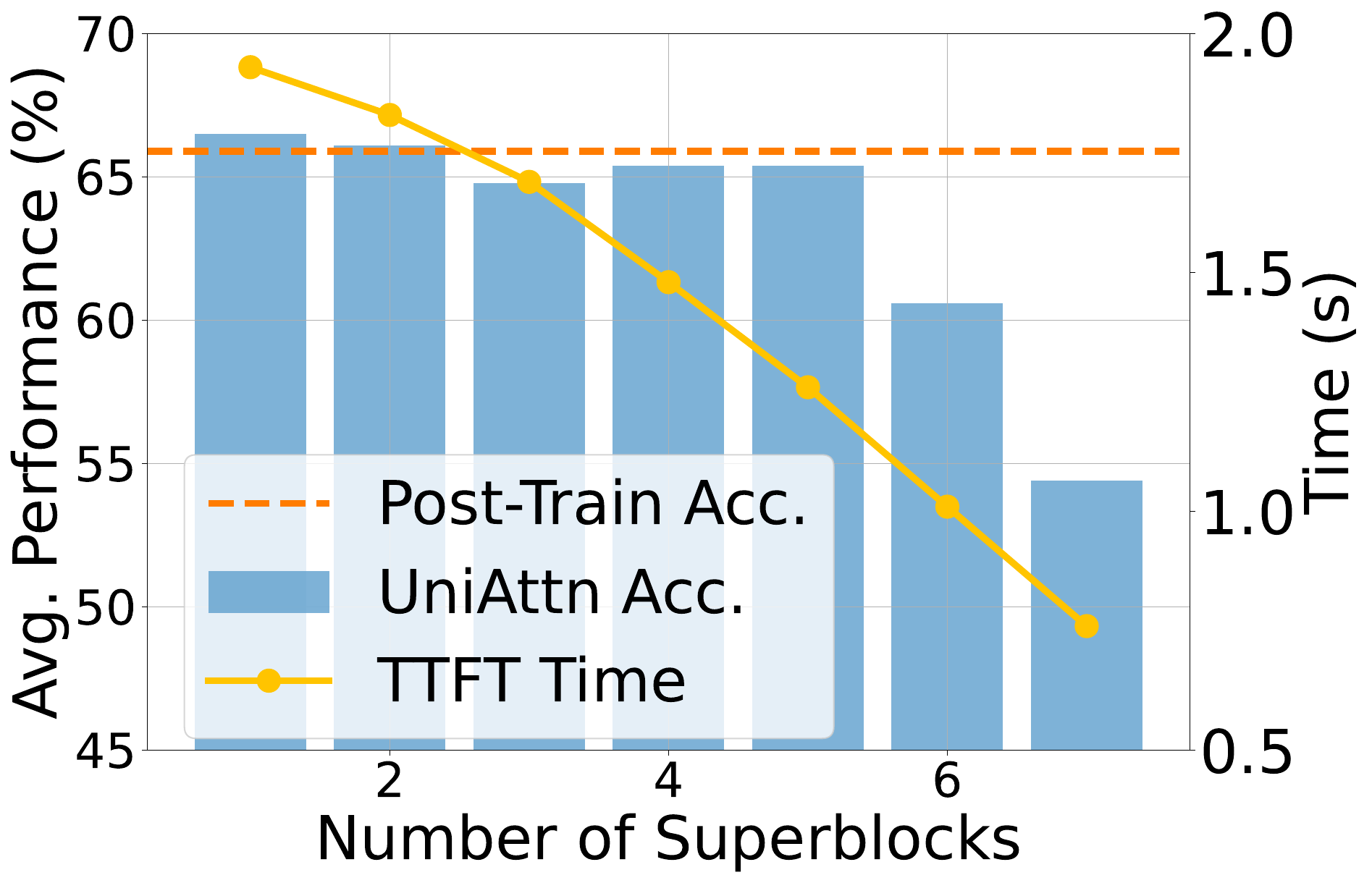}
        \end{minipage}%
        \hfill
        \begin{minipage}{0.49\textwidth}  
            \centering
            \includegraphics[width=\textwidth]{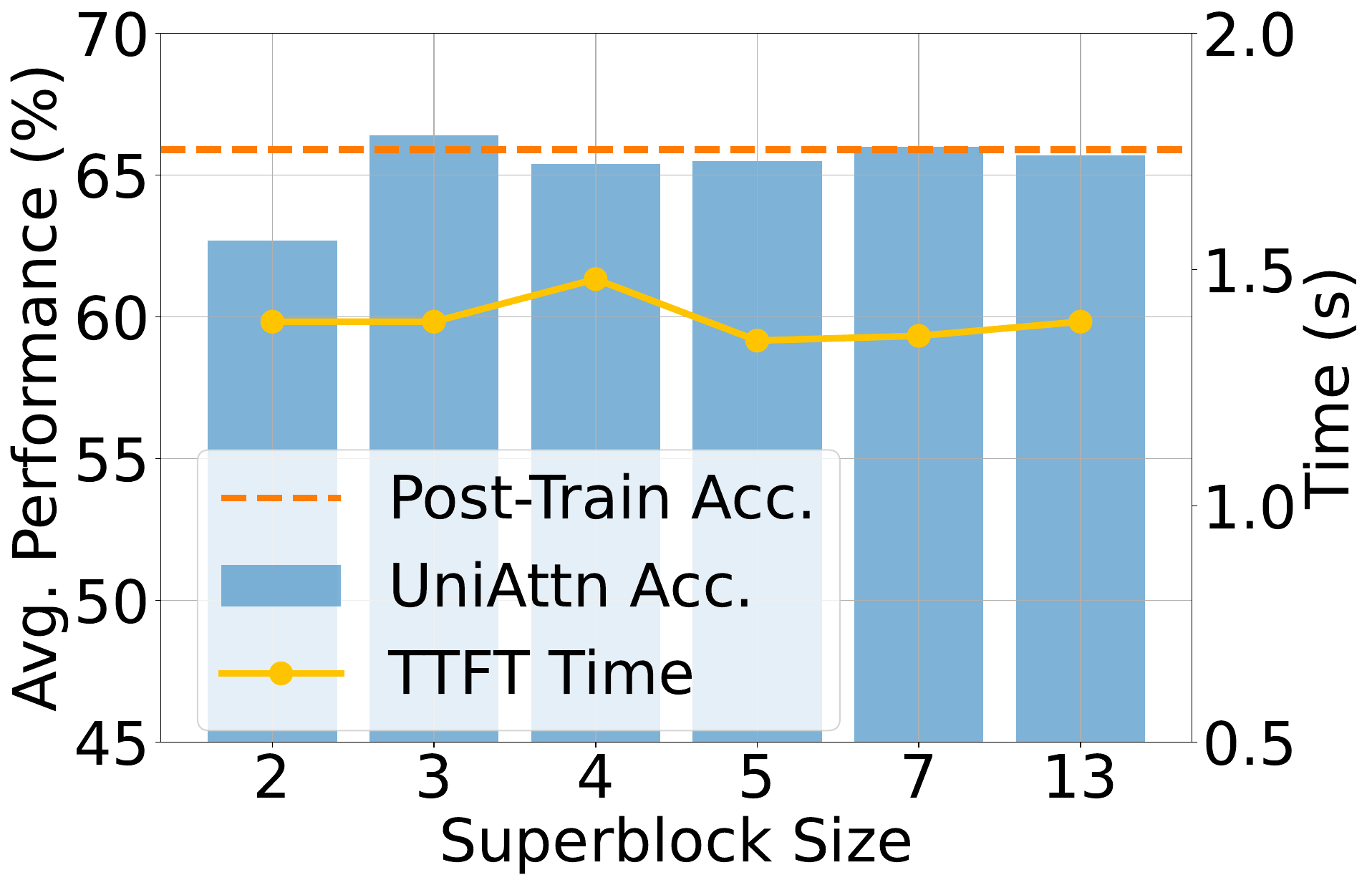}
        \end{minipage}
        \caption{Hyperparameter analysis results on average accuracy (\%) and TTFT latency (s). Left: Different total numbers of grouped Superblocks (each with size 4). Right: Different Superblock sizes (total of 12 layers that utilize unified Softmax activation). }
        \label{fig:hyperparam}
    \end{minipage}%
\end{figure}

\textbf{Ablation Studies.}
We do controlled experiments to evaluate the impact of each component in UniAttn, namely Softmax unification, linear compensation $W_c$, and the initialization for $W_c$. The results are shown in \cref{tab:ablation}.
We observe that after unifying the Softmax activations, both adding linear compensation and using proper initialization contribute positively to the final performance. 
While linear compensation consistently improves performance across different initialization methods, appropriate initialization proves critical for effectively fine-tuning $W_c$.

\textbf{Inference Costs Analysis.}
We evaluate and compare inference costs to demonstrate the efficiency of UniAttn. To further maximize its potential, we integrate UniAttn with the KV-cache compression method H$_2$O \cite{H2O} to further reduce memory overhead by keeping only the essential and the recent tokens in the KV-cache. We report the TTFT latency and the total KV-cache retain rate. As shown in \cref{tab:uniattn_h2o}, UniAttn achieves a significant reduction in TTFT latency compared to CLA under the same KV-cache retain rate.
When combined with H$_2$O, UniAttn+H$_2$O reduces the KV-cache retain rate to below 50\% while maintaining an average performance of 65.0\%, achieving both better performance, lower inference latency, and higher KV-cache compression rate than CLA. 
These results clearly demonstrate the inference efficiency of UniAttn, highlighting its substantial benefits for real-world deployment scenarios.

\textbf{Plain Softmax Unification v.s. Baselines.}
To further substantiate the theoretical analysis and demonstrate the effectiveness of Softmax unification, we compare UniAttn (without linear compensation, UniAttn-w/o-comp) against CLA and LLMDrop. As shown in \cref{fig:wo_proj}, simply unifying Softmax activations yields notable improvements over baselines, highlighting its core contribution. The linear projection serves as an enhancement, bringing UniAttn closer to post-train-level results.

\textbf{Hyperparameter Analysis.}
We perform a hyperparameter analysis on the number and the size of Superblocks to evaluate their impact on latency and post-training performance.
First, we examine the effect of varying the number of Superblocks. Using a consistent Superblock size of 4, we group different numbers of Superblocks from the top layers to the bottom layers. For instance, grouping 2 Superblocks indicates that layer 25-28 and 29-32 form the two Superblocks.
The results are shown in \cref{fig:hyperparam} (left). From the perspective of TTFT latency, increasing the number of Superblocks significantly reduces latency. From the perspective of performance, while grouping certain top layers as Superblocks does not substantially affect model performance, continuing to increase the number of Superblocks leads to sharp performance drops.
Second, we analyze the impact of Superblock size. For fair comparison, we maintain \textit{a consistent total number of layers that utilize unified Softmax activations from preceding layers} as 12. Since a Superblock with size $b$ unifies Softmax in $b-1$ layers, we set $b$ such that $b-1$ is a factor of 12, resulting in Superblock sizes of 2, 3, 4, 5, 7, and 13, and number of Superblocks as 12, 6, 4, 3, 2, and 1, respectively. 
As shown in \cref{fig:hyperparam} (right), the size of the Superblocks does not significantly affect post-training performance. For simplicity, we use a Superblock size of 4, although fine-grained tuning of Superblock size could further improve post-training performance.

\begin{figure}[t]
    \centering
    \includegraphics[width=\linewidth]{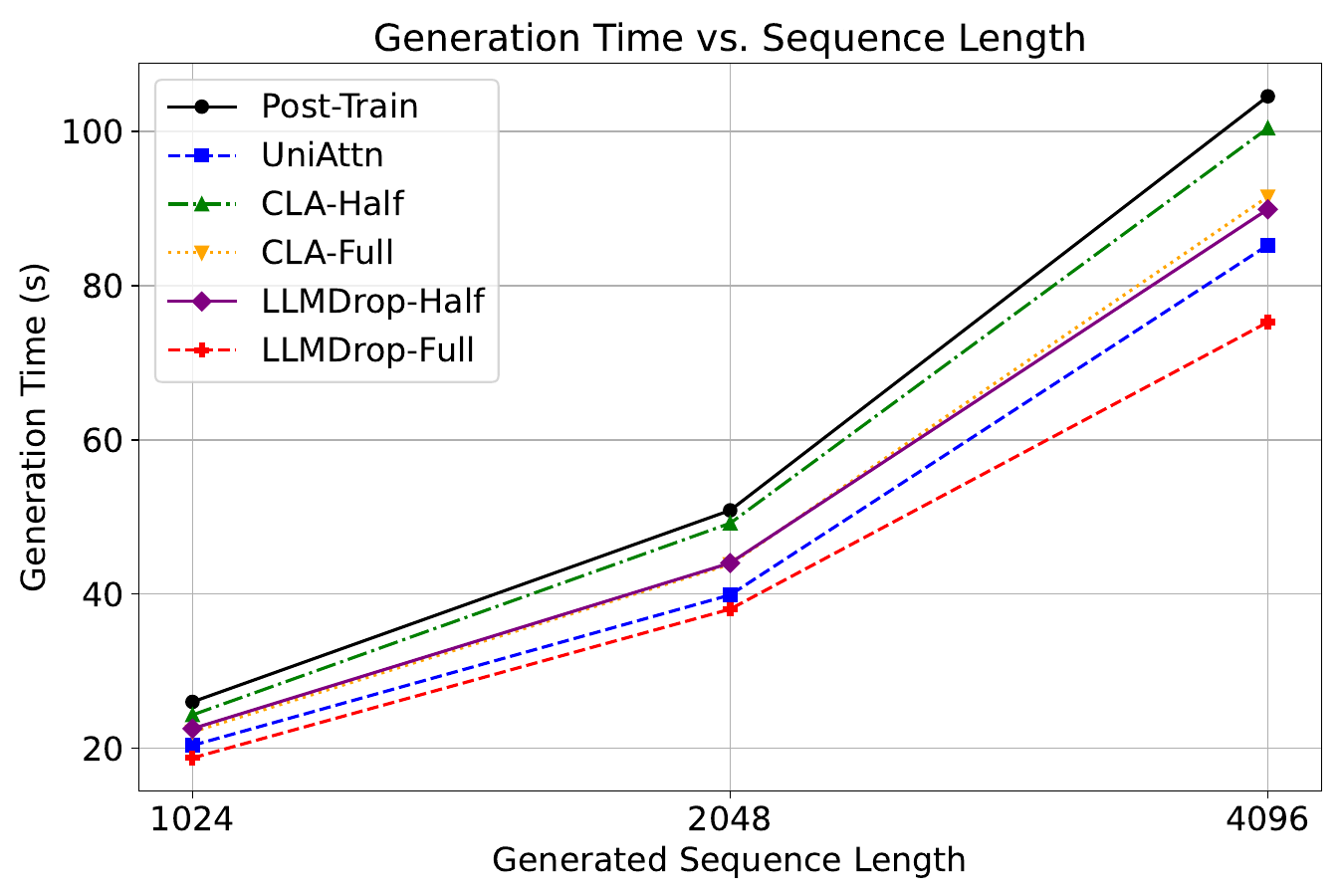}
    \caption{Generation efficiency (wall-clock time seconds) comparison on LLaMA-3.1 8B with 8192 context length.}
    \label{fig:generation}
\end{figure}

\begin{table}[t]
    \centering
    \vskip 0.15in
    \resizebox{\columnwidth}{!}{\begin{tabular}{c|cccc|c}
\toprule
\multirow{2}{*}{Method} & \multicolumn{4}{c|}{Sequence Length} & \multirow{2}{*}{AVG}    \\ 
& 4096   & 8192 & 16384  & 32768 \\ \midrule
Pre-Trained  & 82.67 & 0  & 0   & 0  & 20.67  \\
Post-Train  & \underline{86.67} & \textbf{87}  & \textbf{80.33}   & \textbf{72}  & \textbf{81.5} \\ \hline
LLMDrop-Half  & \textbf{87} & 83.67  & 77.67   & 68.67  & 79.25    \\ 
CLA-Half & \underline{86.67} & 85.67  & 78.67   & 69.33  & 80.08   \\ 
UniAttn (Ours) & \textbf{87} & \underline{86.67}  & \underline{79.67}   & \underline{70.33}  & \underline{80.92}   \\
\bottomrule
 \end{tabular}}
 \caption{Post-training performance (\%) comparison on Niah Multikey in RULER \cite{RULER}. \textbf{Bold} and \underline{underline} denote the best and second-best performance. }
 \label{tab:ruler}
\vskip -0.1in
\end{table} 

\begin{table}[t]
    \centering
    \resizebox{\columnwidth}{!}{\begin{tabular}{cc|ccc|c}
\toprule
Method & Grouping & PubMed & MedMCQA & MedQA & AVG \\
\midrule
Post-Train & - & 75.4 & 59.0  & 63.2 & 65.9 \\ \hline 
\multirow{2}{*}{UniAttn} & Fixed size (4) & \textbf{79.0} & \textbf{57.6} & \underline{59.5} & \textbf{65.4} \\
 & Adaptive & 78.8 & 57.3 & \textbf{59.5} & 65.2 \\
\bottomrule
\end{tabular}}
\caption{Performance (\%) comparison of different Superblock grouping methods in UniAttn. \textbf{Bold} denotes the best performance.}
\label{tab:grouping}
\end{table}

\textbf{Generation Efficiency Analysis.}
To further analyze inference efficiency, we present the generation results of UniAttn against the comparing baselines, namely CLA-Half, CLA-Full, LLMDrop-Half, and LLMDrop-Full. We measure the end-to-end generation time with the LLaMA-3.1 8B backbone, and with generating 1024, 2048, and 4096 tokens using a context length of 8192. As shown in \cref{fig:generation}, across all comparing methods, UniAttn consistently achieves Top-2 latency, only slightly behind LLMDrop-Full where all operated layers are simply skipped. These results demonstrate that unifying Softmax operations can substantially reduce the inference overhead during generation.

\textbf{Generalization to Long-Context Modeling.}
To further validate the effectiveness of UniAttn on long context modeling scenarios, we follow the training recipe of \cite{LongAlign} and post-train LLaMA-2 7B, which lacks native support for extended contexts. We compare UniAttn with two strong baselines, CLA-Half and LLMDrop-Half, on the Niah Multikey task in the RULER benchmark \cite{RULER}. To successfully implement long-context modeling, here we conduct experiments \textit{with FlashAttention} \cite{Flashattention} applied. As shown in \cref{tab:ruler}, UniAttn consistently ranks in the top two (second only to full post-training) and outperforms all competing methods, well demonstrating its strong generalization ability to complex long-context tasks. This also suggests that our UniAttn is compatible with common acceleration frameworks such as FlashAttention.

\textbf{Adaptive Superblock Grouping Schedule.}
To further investigate layerwise redundancy–aware grouping, we performed an additional experiment using a similarity-based adaptive scheme. We controlled the total number of ``grouped layers'' and assigned Superblocks by cosine similarity of layerwise Softmax activations. As shown in \cref{tab:grouping}, for LLaMA-3.1 8B on the medical domain, this similarity-based grouping variant achieved comparable performance to setting the Superblock size as a fixed number, indicating that the performance of UniAttn is robust to different grouping strategies, and that adaptive grouping does not yield meaningful gains over the simpler fixed-size design.


\section{Conclusion}

In this work, we explored LLM redundancies and efficient architectures in the context of post-training to reduce the inference costs of LLMs. 
While existing KV-cache sharing methods successfully reduce the memory overhead, they yield limited inference time reduction during deployment.
To address this limitation, we investigated the primary bottleneck in LLMs, the Softmax operation, and discovered significant redundancies across various pre-trained LLMs and datasets. Based on these findings, we proposed UniAttn, an efficient architecture that leverages these redundancies while preserving LLM capabilities through linear compensation in post-training.
We further support UniAttn with a theoretical analysis, demonstrating its advantage over baselines.
Extensive experiments on a series of models and datasets validate the effectiveness of UniAttn. 
Our UniAttn is particularly well-suited for real-world deployments, enjoying faster inference and less memory cost. 

\section*{Impact Statement}

This paper presents work whose goal is to advance the field of 
Machine Learning. There are many potential societal consequences 
of our work, none of which we feel must be specifically highlighted here.

\bibliography{example_paper}
\bibliographystyle{icml2026}

\newpage
\appendix
\onecolumn
\begin{table*}[h]
    \centering
    \vskip 0.15in
    \resizebox{0.95\textwidth}{!}{\begin{tabular}{c|c|cccccc}
\toprule
Method & Model & Learning Rate & Weight Decay & Batch Size & Epochs & Superblock Size & Superblock Groups \\
\midrule
\multirow{4}{*}{UniAttn} & LLaMA-2 7B & 2e-5 & 0 & 48 & 1 & 4 & [17-20], [21-24], [25-28], [29-32] \\
 & LLaMA-3.1 8B & 7e-6 & 0 & 48 & 1 & 4 & [17-20], [21-24], [25-28], [29-32] \\
 & Mistral 7B & 1e-6 & 0 & 48 & 1 & 4 & [17-20], [21-24], [25-28], [29-32] \\
 & Gemma-2 9B & 1e-6 & 0 & 48 & 1 & 4 & [22-25], [26-29], [30-33], [34-37], [38-41] \\
\bottomrule
\end{tabular}}
\caption{Training hyperparameter details of UniAttn.}
\label{tab:hyperparam_uniattn}
\vskip -0.1in
\end{table*}

\begin{table*}[h]
    \centering
    \vskip 0.15in
    \resizebox{0.95\textwidth}{!}{\begin{tabular}{c|c|cccccc}
\toprule
Method & Model & Learning Rate & Weight Decay & Batch Size & Epochs & CLA Block Size & CLA Block Groups \\
\midrule
\multirow{4}{*}{CLA-Half} & LLaMA-2 7B & 2e-5 & 0 & 48 & 1 & 4 & [25-28], [29-32] \\
 & LLaMA-3.1 8B & 7e-6 & 0 & 48 & 1 & 4 & [25-28], [29-32] \\
& Mistral 7B & 1e-6 & 0 & 48 & 1 & 4 & [25-28], [29-32] \\
& Gemma-2 9B & 1e-6 & 0 & 48 & 1 & 4 & [30-33], [34-37], [38-41] \\ \midrule
\multirow{4}{*}{CLA-Full} & LLaMA-2 7B & 2e-5 & 0 & 48 & 1 & 4 & [17-20], [21-24], [25-28], [29-32] \\
 & LLaMA-3.1 8B & 7e-6 & 0 & 48 & 1 & 4 & [17-20], [21-24], [25-28], [29-32] \\
& Mistral 7B & 1e-6 & 0 & 48 & 1 & 4 & [17-20], [21-24], [25-28], [29-32] \\
& Gemma-2 9B & 1e-6 & 0 & 48 & 1 & 4 & [22-25], [26-29], [30-33], [34-37], [38-41] \\
\bottomrule
\end{tabular}}
\caption{Training hyperparameter details of CLA.}
\label{tab:hyperparam_cla}
\vskip -0.1in
\end{table*}

\begin{table*}[h]
    \centering
    \vskip 0.15in
    \resizebox{0.95\textwidth}{!}{\begin{tabular}{c|c|cccccc}
\toprule
Method & Model & Learning Rate & Weight Decay & Batch Size & Epochs & \# of Dropped Layers  & Index of Dropped Layers \\
\midrule
\multirow{4}{*}{LLMDrop-Half} & LLaMA-2 7B & 2e-5 & 0 & 48 & 1 & 6 & 23,20,32,27,22,24 \\
 & LLaMA-3.1 8B & 7e-6 & 0 & 48 & 1 & 6 & 19,22,20,21,27,23 \\
& Mistral 7B & 1e-6 & 0 & 48 & 1 & 6 & 17,32,31,21,22,24 \\
& Gemma-2 9B & 1e-6 & 0 & 48 & 1 & 8 & 29,26,21,30,27,40,38,31 \\ \midrule
\multirow{4}{*}{LLMDrop-Full} & LLaMA-2 7B & 2e-5 & 0 & 48 & 1 & 12 & 23,20,32,27,22,24,31,29,25,30,28,26 \\
 & LLaMA-3.1 8B & 7e-6 & 0 & 48 & 1 & 12 & 19,22,20,21,27,23,30,29,28,26,24,25 \\
& Mistral 7B & 1e-6 & 0 & 48 & 1 & 12 & 17,32,31,21,22,24,29,23,28,25,27,26 \\
& Gemma-2 9B & 1e-6 & 0 & 48 & 1 & 15 & 29,26,21,30,27,40,38,31,32,33,35,39,36,37,34 \\
\bottomrule
\end{tabular}}
\caption{Training hyperparameter details of LLMDrop. The index of dropped layers are ordered by input-output similarity.}
\label{tab:hyperparam_llmdrop}
\vskip -0.1in
\end{table*}

\begin{table}[!t]
\centering
\scalebox{1.0}{\begin{tabular}
{lc|lc}
\toprule
 Package & Version & Package & Version \\
\midrule
PyTorch & 2.0.0 & transformers & 4.46.0 \\
deepspeed & 0.10.0 & tokenizers & 0.20.1 \\
datasets & 2.14.3 &  &  \\
\bottomrule
\end{tabular}}
\caption{Versions of used packages.}
\label{tab:package_version}
\end{table}

\section{License for Scientific Artifacts}

The Training dataset PMC \cite{PMC} is licensed under OpenRAIL license\footnote{https://www.licenses.ai/blog/2022/8/26/bigscience-open-rail-m-license}. The Training dataset TuLu3 \cite{TULU3} is licensed under Open Data Commons License\footnote{https://opendatacommons.org/licenses/}. The Mistral model \cite{Mistral}, Gemma-2 model \cite{Gemma2}, LLaMA-2 \cite{LawyerLLaMA} model, and LLaMA-3.1 model\cite{LLaMA3} are licensed under Apache License 2.0\footnote{https://choosealicense.com/licenses/apache-2.0/}. The LongAlign evaluation dataset is also licensed under Apache License 2.0\footnote{https://choosealicense.com/licenses/apache-2.0/}. The evaluation datasets \cite{RULER,CommonsenseQA,SIQA,MedQA,MedMCQA,PubMedQA} are subject to the MIT license\footnote{https://choosealicense.com/licenses/mit/}. All usages of scientific artifacts in this paper obey the corresponding licenses.

\section{More Implementation Details for Reproducibility}

\textbf{Post-Training Datasets.}
Regarding the medical domain, we directly employ the instruction tuning dataset for PMC\footnote{https://huggingface.co/datasets/axiong/pmc\_llama\\\_instructions} \cite{PMC} (i.e., training stage-2 for PMC-LLaMA) for post-training. The dataset consists of 513,999 instruction-input-output pairs.
Regarding the general instruction tuning scenario, we employ the Tulu 3 SFT Mixture\footnote{https://huggingface.co/datasets/allenai/tulu-3-sft-mixture} and filter for data with only 1 round of conversation, resulting in 896,070 input-output pairs.

\textbf{Training Hyperparameters.}
The training hyperparameters for all experiments are reported in table \cref{tab:hyperparam_uniattn}, \cref{tab:hyperparam_cla} and \cref{tab:hyperparam_llmdrop}.
Note that we adopt the learning rate and weight decay values according to existing research, in which those pre-trained models are post-trained under reported schedules.

\textbf{Evaluation Details.}
For all experiments, we evaluate the final checkpoint after post-training for benchmark evaluation. 
Regarding the medical domain, we report the 0-shot result on both PubMed, MegMCQA, and MedQA datasets. Regarding the general instruction tuning scenario, we report 0-shot result on both SIQA and CommonsenseQA for LLaMA-3.1 8B, Mistral 7B, and Gemma-2 9B models. For the earlier-released LLaMA-2 7B model, since it has a shorter pre-training context length, we report its 5-shot result as Tulu 3 mainly enhances model performance with longer context.

\textbf{Prompt for Post-Training.}
We employ the same prompt for post-training in both the medical domain and general instruction tuning scenario:
\begin{verbatim}
    "Below is an instruction that 
    describes a task, paired with an input 
    that provides further context. Write a 
    response that appropriately completes 
    the request.\n\n### 
    Instruction:\n{instruction}\n\n### 
    Input:\n{input}\n\n### 
    Response:{output}<EOS_TOKEN>"
\end{verbatim}
We simply keep the instruction field empty if no instruction is provided.

\textbf{Experiment Packages.}
We report the version numbers of used packages in Table \ref{tab:package_version}.

\section{More Experiment Results}
\label{sec:append_exp}

\textbf{Main Results on Mistral 7B and Gemma-2 9B.}
We adopt the settings in \cref{tab:main} and conduct the same experiments on Mistral 7B \cite{Mistral} and Gemma-2 9B \cite{Gemma2}. The results are shown in \cref{tab:main_appendix_g}.
Consistent with the results on the LLaMA models, our UniAttn achieves the best overall performance on both LLMs and post-training datasets while significantly reducing both time and memory costs. When operating the same number of layers, UniAttn achieves similar TTFT time to LLMDrop, well demonstrating the effectiveness of Softmax activation unification. Although CLA and LLMDrop achieve competitive results in some settings (e.g., CLA-Half in post-training Gemma on general instruction datasets), \textit{they cannot provide consistent performance across different settings}. 
Additionally, even on the Mistral and Gemma models, CLA achieves similar performance to its corresponding LLMDrop setting, again indicating that CLA diminishes the depth factor in pre-trained LLMs.

\textbf{Quantitative Analysis of Softmax Redundancies.}
In \cref{fig:softmax_sim} of the main article, we observed that Softmax activations in \textit{top half layers} share a high cross-layer similarity across different models and datasets. To provide a more comprehensive analysis, here we take LLaMA-3.1 8B as an example and present some statistics of the cosine similarity of Softmax activations. As shown in \cref{tab:quantitative_similarity}, cosine similarities appear to be consistent across different datasets, with average values ranging from 0.79 to 0.83. The maximum cosine similarity remains consistently high at 0.97, indicating that the overall range of cosine similarity remains stable across layers and datasets. The differences for cosine similarity between the top and bottom halves of layers indicate that Softmax operations are more redundant in deeper layers, as noted in the main article. The standard deviation values suggest that redundancy is fairly consistent across datasets. Overall, these results indicate that Softmax activations for pre-trained LLMs exhibit high redundancies.

\begin{table*}[t]
    \centering
    \resizebox{0.94\textwidth}{!}{\begin{tabular}{c|cc|ccc|c|cc|c}
\toprule
\multirow{2}{*}{Method} & \multirow{2}{*}{\makecell{TTFT \\ (s)}} & \multirow{2}{*}{\makecell{KV \\ Cache}} & \multicolumn{4}{c|}{Medical} &  \multicolumn{3}{c}{General}    \\ \cline{4-7} \cline{8-10} 
&  & & PubMedQa   & MedMCQA  & \multicolumn{1}{c}{MedQA} & AVG & SIQA  & \multicolumn{1}{c}{CommonsenseQA} & AVG  \\ \midrule
\multicolumn{10}{c}{\textbf{Mistral-7B (w/ GQA)}}  \\ 
\midrule
Pre-Trained & 1.94 & 100\%  & 75.8   & 48.3  & 50.9   & 58.3 & 46.7 & 56.5  & 51.6 \\
Post-Train & 1.94 & 100\%  & 78.8   & 49.7  & 60.3   & 62.9 & 53.7 & 76.6  & 65.2 \\ \hline
LLMDrop-Half & 1.54 & 81.3\%  & 78.0   & 50.6  & \underline{59.0}   & 62.5 & \underline{52.5} & 74.5  & 63.5 \\ 
LLMDrop-Full & 1.21 & 62.5\%  & \textbf{79.0}   & 50.2  & 56.7   & 62.0 & 51.8 & 73.6  & 62.7 \\ 
CLA-Half & 1.94 & 81.3\%  & 78.6   & \underline{56.0}  & 58.5   & \underline{64.4} & 52.0 & \textbf{77.9}  & \underline{65.0} \\ 
CLA-Full & 1.91 & 62.5\%  & \underline{78.8}   & 55.0  & 55.0   & 62.9 & 50.5 & 75.1  & 62.8 \\ 
UniAttn & 1.23 & 81.3\%  & 78.2   & \textbf{57.3}  & \textbf{60.5}   & \textbf{65.3} & \textbf{53.4} & \underline{77.5}  & \textbf{65.5} \\ 
\midrule
\multicolumn{10}{c}{\textbf{Gemma2-9B (w/ GQA)}}  \\ 
\midrule 
Pre-Trained & 2.23 & 100\%  & 78.6   & 57.6  & 60.6   & 65.6 & 51.5 & 68.4  & 60.0 \\
Post-Train & 2.23 & 100\%  & 78.6   & 58.6  & 61.9   & 66.4 & 54.4 & 75.6  & 65.0 \\ \hline
LLMDrop-Half   & 1.94 & 81.0\%  & 78.6   & 56.4  & \underline{60.5}   & \underline{65.2} & \textbf{55.6} & 68.4  & 62.0 \\ 
LLMDrop-Full & 1.66 & 64.3\%  & 78.2   & 54.8  & 58.7   & 63.9 & \underline{54.8} & 62.9  & 58.9 \\ 
CLA-Half & 2.19 & 78.6\%  & \textbf{79.2}   & \underline{56.6}  & 56.3   & 64.0 & 53.5 & \textbf{76.0}  & \textbf{64.8} \\ 
CLA-Full & 2.24 & 64.3\%  & 74.0   & 39.3  & 38.5   & 50.6 & 48.3 & 51.8  & 43.4 \\ 
UniAttn & 1.69 & 82.1\%  & \underline{79.0}   & \textbf{60.7}  & \textbf{64.3}   & \textbf{68.0} & 53.7 & \underline{74.4}  & \underline{64.1} \\ 
\bottomrule
\end{tabular}}
\caption{Post-training performance comparison. \textbf{Bold} and \underline{underline} denote the best and second-best performance of compressed models. For each method, we report their time to first token (TTFT, in seconds) and KV-cache retain rate (KV Cache).}
\label{tab:main_appendix_g}
\end{table*}

\begin{table*}[h]
    \centering
    \vskip 0.15in
    \resizebox{0.85\textwidth}{!}{\begin{tabular}{l|ccccc}
\toprule
Dataset & Average & Difference Between Top Half and Bottom Half Layers & Max & Min & Standard Deviation \\
\midrule
PMC & 0.8347 & 0.1314 & 0.9729 & 0.0159 & 0.1955 \\
Math & 0.8243 & 0.1149 & 0.9734 & 0.0384 & 0.1860 \\
Tulu3 & 0.8214 & 0.1212 & 0.9719 & 0.0307 & 0.1896 \\
F15K & 0.7927 & 0.1475 & 0.9710 & 0.0033 & 0.2109 \\
\bottomrule
\end{tabular}}
\caption{Statistics of the cosine similarity of Softmax activations for LLaMA-3.1 8B across different datasets.}
\label{tab:quantitative_similarity}
\vskip -0.1in
\end{table*}


\begin{table}[t]
\centering
\scalebox{0.82}{\begin{tabular}
{cc|c}
\toprule
Input to $W_c$ & Output adds to & AVG \\
\midrule
 MHA input & MHA output & \textbf{65.4} \\
 MHA input & FFN output & 64.1 \\  
 MHA input & Activation before $W_o$ & \underline{65.1} \\ 
 Activation projected by $W_v$ & Activation before $W_o$ & 64.3 \\ 
\bottomrule
\end{tabular}}
\caption{Performance comparison (\%) of different compensation designs.}
\label{tab:compensation_design}
\end{table}

\textbf{Empirical Statistics for the ``Depth Factor''.}
To further support \cref{assum:inftesimal} and the claim made in \cref{prop:effective_depth}, we applied cross-layer sharing to the top layers of LLaMA-3.1 8B and evaluated 100 randomly selected samples from the PMC post-training corpus. The average results of the following terms are: $||\delta_i||/||\mathbf{x}_i||=0.070$, $||J(A_i)||=0.041$, and $||J(A_i)\hat{\delta}||/||\text{Softmax}(A_i)||=0.038$. These results support that (1) $\delta_i$ is small comparing to $||\mathbf{x}_i||$, validating Assumption 3.5; (2) the depth factor is indeed diminished under cross-layer KV-cache sharing, consistent with Proposition 3.6.

\textbf{Different Compensation Designs.}
To compare different compensation designs, we attach the input and output of the $W_c$ transformation to different activations. As shown in \cref{tab:compensation_design}, directly compensating the MHA output with its input yields the best results. When adopting ``finer compensation granularities'', i.e., positioning the input and output closer together (Rows 3 and 4), post-training performance decreases. Similarly, adopting ``coarser compensation granularities'', i.e., moving the output compensation to after the FFN (Row 2), also results in lower performance. These findings suggest that using a linear transformation to compensate for errors specifically within the MHA module is the best-performing approach. 
\begin{figure}[!t]  
\vskip 0.2in
\begin{center}
\centerline{\includegraphics[width=0.6\columnwidth]{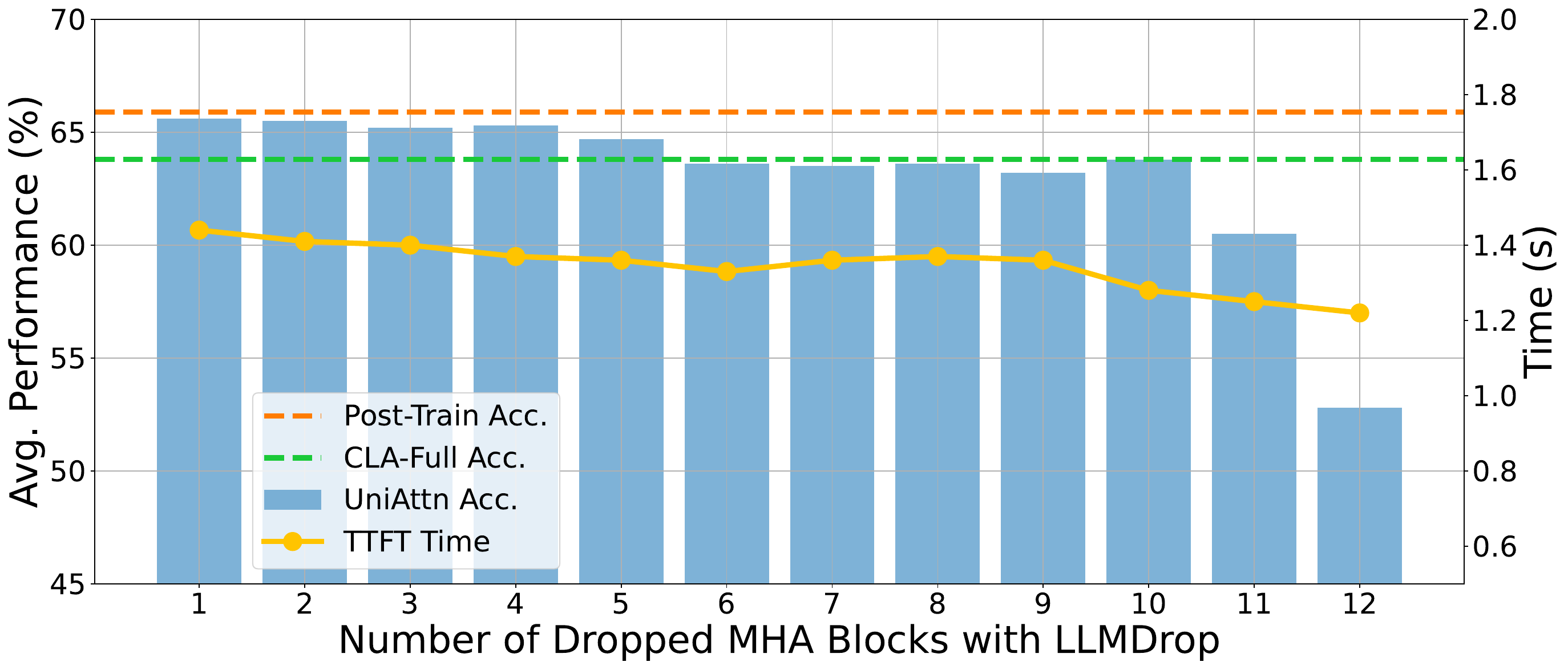}}
    \caption{Results on further dropping UniAttn layers (\%). We present the average accuracy and time to first token (TTFT) latency under different number of dropped MHA blocks via LLMDrop. Note that we only drop MHA blocks with unified Softmax activations. }
    \label{fig:llmdrop}
\end{center}
\vskip -0.2in
\end{figure}

\begin{algorithm}[!t]
   \caption{UniAttn Pipeline}
   \label{alg:example}
\begin{algorithmic}
   \STATE {\bfseries Input:} Pre-trained Model $M$, Post-train dataset $D$, SuperBlock size $b$, Apply SuperBlock merging from layer $i_s$ to $i_e$
   \STATE Create UniAttn model $M_{\text{uni}}$ by creating $\frac{i_e-i_s+1}{b}$ SuperBlocks in $M$
   \STATE Sample a subset $D_{\text{init}}$ with 1000 samples from $D$
   \STATE For model $M$, forward all samples from $D_{\text{init}}$, calculate the average output activations $\mathbf{x}_i'$ for each MHA block.
   \FOR{$i=1$ {\bfseries to} $\frac{i_e-i_s+1}{b}$}
   \FOR{$j=1$ {\bfseries to} $b-1$}
   \STATE For model $M_{\text{uni}}$, forward all samples from $D_{\text{init}}$, calculate the average output activations $\mathbf{x}_{\text{uni,}i_s+(i-1)b+j}'$ and the average input activations $\mathbf{x}_{\text{uni,}i_s+(i-1)b+j}$ for MHA block in layer $i_s+(i-1)b+j$.
   \STATE Calculate the initialization for the $W_{c,i_s+(i-1)b+j}$ matrix in layer $i_s+(i-1)b+j$ according to \cref{theorem:init}.
   \STATE Insert the initialized $W_{c,i_s+(i-1)b+j}$ matrix to layer $i_s+(i-1)b+j$ in model $M_{\text{uni}}$.
   \ENDFOR
   \ENDFOR
   \STATE Freeze all training parameters other than the $W_c$ matrices in $M_{\text{uni}}$. Train $W_c$ with dataset $D$ and apply early stop when loss stops decreasing (we consider the exponential moving average of training loss).
   \STATE Conduct full fine-tuning on $M_{\text{uni}}$ with $D$.
   \STATE {\bfseries Output:} Trained UniAttn model $M_{\text{uni}}$.
\end{algorithmic}
\end{algorithm}
\textbf{Experimental Analysis on the Impact of Model Depth.} 
In \cref{sec:discussion}, we demonstrated that, unlike cross-layer KV sharing methods, UniAttn does not diminish the depth factor in pre-trained LLMs. To provide experimental evidence for this, we further apply LLMDrop to the MHA modules that utilize unified Softmax activations from preceding layers (a total of 12 layers in UniAttn for the LLaMA-3.1 8B model) in post-trained UniAttn models.
Specifically, we compute the input-output similarities of these MHA modules and simply prune the top $k$ MHA modules with the highest similarities. The pruned model is then directly evaluated without additional fine-tuning. The average performance w.r.t. $k$ in shown in \cref{fig:llmdrop}.
As shown in the figure, pruning a few MHA modules from our UniAttn model still results in higher performance than the CLA-Full model variant.
After pruning 10 MHA modules, UniAttn+LLMDrop achieves an average performance of 63.8, matching the performance of CLA-Full reported in \cref{tab:main}. This demonstrates that UniAttn better preserves the impact of model depth compared to CLA. Furthermore, we can also conclude that the depth of the model significantly contributes to the post-train performance, thereby validating our theoretical analysis.


\section{Pipeline for Applying UniAttn during Post-Training}

We provide a detailed pipeline for applying UniAttn as \cref{alg:example}.

\newpage

\section{Theoretical Insights for Proposing \cref{assum:inftesimal}}
\label{append:theo_insights}

We first give a series of conclusions (namely \cref{lem:linear}, \cref{prop:linear}, \cref{prop:similarity}, and \cref{prop:prenorm}) as preliminaries.

\begin{lemma}
\label{lem:linear}
(\textbf{A linear Pre-Norm system has bounded growth})
$f_i:\mathbb{R}^n\rightarrow\mathbb{R}^n$, $i\in \mathbb{N}$ are LINEAR transformations and $\text{Norm}(\cdot)$ yields a vector with unit Frobenius-norm (denoted as $||\cdot||$). Let $\mathbf{x}_0\in \mathbb{R}^n$ be the input to the system. Consider an $L$-layer Pre-Norm architecture defined by:
\begin{equation}
\begin{aligned}
    &\mathbf{x}_{k}=\mathbf{x}_{k-1}+f_k(\text{Norm}(\mathbf{x}_{k-1})),\\&\quad \text{for}\quad k=1,2,\dots,L.
\end{aligned}
\end{equation} 
If the largest singular value in all transformation matrices is bounded by $\lambda$, then:
\begin{equation}
    ||\mathbf{x}_L||\le||\mathbf{x}_0||+\lambda L
\end{equation} 
\end{lemma}

\begin{proof} 
Unroll the expression of $\mathbf{x}_L$ and inductively:
\begin{equation}
    \mathbf{x}_L=\mathbf{x}_0+\sum_{k=1}^{L}F_k(\text{Norm}(\mathbf{x}_{k-1}))
\end{equation}
As $||\text{Norm}(\mathbf{x}_k)||=1$ for all $k$, then:
\begin{equation}
    \begin{aligned}
        ||\mathbf{x}_L|| &= ||\mathbf{x}_0+\sum_{k=1}^{L}F_k(\text{Norm}(\mathbf{x}_{k-1}))|| \\
        &\le ||\mathbf{x}_0||+\sum_{k=1}^{L}||F_k(\text{Norm}(\mathbf{x}_{k-1}))|| \\
        &\le ||\mathbf{x}_0||+\sum_{k=1}^{L}||\lambda\cdot\text{Norm}(\mathbf{x}_{k-1}))|| \\
        &= ||\mathbf{x}_0||+\lambda L
    \end{aligned}
\end{equation}
\end{proof}

\cref{lem:linear} shows that a Pre-Norm linear system has bounded growth. The norm of output by each layer grows \textit{at most linearly} with depth $L$. 
\begin{proposition}
\label{prop:linear}
    (\textbf{Pre-trained LLMs are linear Pre-Norm systems})
    Pre-trained decoder-based LLMs exhibit a high linearity score. Formally, let $A, B\in \mathbb{R}^{n\times d}$ denote the normalized input and output of a decoder block in LLM, respectively,
\begin{equation}
    \min_X ||AX-B||\approx1
\end{equation}
\end{proposition}

\cref{prop:linear} has been validated by \cite{razzhigaev-etal-2024-transformer} on pre-trained LLMs, 
suggesting that we can approximate each decoder block by a linear transformation with bounded singular values, 
thus placing pre-trained LLMs structurally “close” to a linear Pre-Norm system as in \cref{lem:linear}.
In the following discussion, we approximate pre-trained LLMs with an equivalent linear Pre-Norm system and investigate its input and output in each layer. 

\begin{proposition}
\label{prop:similarity}
    (\textbf{Pre-trained LLM layers primarily change the magnitude of activations})
    The input and output of a layer in pre-trained decoder-based LLMs share a high cosine similarity.
\end{proposition}
\cref{prop:similarity} has been validated by \cite{ShortGPT} on pre-trained LLMs, indicating that each layer in LLMs predominantly change the \textit{magnitude} of the activation.
We leverage the following conclusion for analyzing the change of activation magnitude.
\begin{proposition}
\label{prop:prenorm}
    (\textbf{Top layers of Pre-Norm transformers have smaller singular values})
    Training Pre-Norm transformers leads to top layers receiving smaller updates (gradient norms decrease from bottom to top).
\end{proposition}
\cite{DBLP:conf/icml/XiongYHZZXZLWL20} has given both a theoretical proof sketch and experimental evidence of \cref{prop:prenorm}. 
Based on \cref{prop:linear}, we approximate each layer in LLMs as a linear transformation.
To further show that the approximated linear transformations of the top layers in Pre-Norm LLMs tend to have smaller singular values, we give the following proof:

\begin{proof} 
With \cref{lem:linear} and \cref{prop:linear}, we have shown that pre-trained Pre-Norm LLMs can be approximated as linear Pre-Norm systems. Suppose a pre-trained LLM operates on activations $\mathbf{x}\in \mathbb{R}^{d}$, each layer $i$ in the pre-trained LLM can be treated as a linear transformation matrix $W_i\in \mathbb{R}^{d\times d}$. A generic gradient-based update can be expressed as:
\begin{equation}
    W_i^{t+1} = W_i^0 - \sum_{t=0}^{t}\eta\nabla_{W_i}\mathcal{L}^{t}
\end{equation}
Hence, the norm between initial and final weights satisfies:
\begin{equation}
    ||W_i^{t+1}-W_i^0|| \le \sum_{t=0}^{t}||\eta\nabla_{W_i}\mathcal{L}^{t}||
\end{equation}
We denote the singular value decomposition of $W_i^{t+1}-W_i^0$ as $W_i^{t+1}-W_i^0=U\Sigma V^T$, it is easy to derive that:
\begin{equation}
\begin{aligned}
    ||&W_i^{t+1}-W_i^0||\\ &= \sqrt{tr[(||W_i^{t+1}-W_i^0||)(||W_i^{t+1}-W_i^0||)^T]}\\&=\sqrt{tr(U\Sigma V^TV\Sigma^TU^T)} \\ &= \sqrt{tr(\Sigma\Sigma^T)}=\sqrt{\sum_{i=1}^{d}\sigma_i^2}
\end{aligned}
\end{equation}
where $\sigma_i$ denote the $i$-th singular value in $\Sigma$. Normally, the largest singular value $\sigma_{\max}$ dominates in the quadratic term, thus we can further write:
\begin{equation}
    ||W_i^{t+1}-W_i^0||=\sqrt{\sum_{i=1}^{d}\sigma_i^2}\approx \sigma_{\max}(W_i^{t+1}-W_i^0)
\end{equation}
Based on the basic features of singular values, we can show that:
\begin{equation}
\begin{aligned}
    \sigma_{\max} (W_i^{t+1})&\le\sigma_{\max}(W_i^0)+\sigma_{\max} (W_i^{t+1}-W_i^0)\\&\le\sum_{t=0}^{t}||\eta\nabla_{W_i}\mathcal{L}^{t}||+\sigma_{\max} (W_i^0)
\end{aligned}
\end{equation}
Due to having smaller values of $||\eta\nabla_{W_i}\mathcal{L}^{t}||$ (i.e., smaller gradient norm), compared to earlier layers, the top layers typically have smaller updates on the largest singular value, thus tend to have smaller $\sigma_{\max}$ values.
\end{proof}

Since pre-trained LLMs mainly change the magnitude of the activation (as shown in \cref{prop:similarity}), we can use the size of the largest singular values in each layer as an indicator of the corresponding activation magnitude change. In Pre-Norm architectures, each layer operates on an input with a constant norm. This leads to the assumption that the top-layer transformations generally generate outputs with smaller norms, which is \cref{assum:inftesimal} we propose in the article.


\section{Proof for \cref{theorem:init}}
\begin{proof}
Solving for the initialization of $W_c$ that incurs minimal compensation error is equivalent to solving the following optimization problem:
\begin{equation}
    \min_{W_c}\ ||\mathbb{E}(\mathbf{x}_{i+b}W_c-\epsilon)||.
\end{equation}
Since we use a small subset of the training dataset (with $S$ samples, we denote this as the \textit{initialization set}) to calculate the initialization, the objective above can be rewritten as:
\begin{equation}
    \min_{W_c}\ \frac{1}{S}\sum_{j=1}^S||\mathbf{x}_{i+b,j}W_c-\epsilon_j||||.
\end{equation}
This is a Least Absolute Deviations problem, and can be solved with the simplex algorithm, which is NP as shown in \cite{fearnley2015complexity}. To guarantee that the initialization method is efficient even in the worst case, we in turn consider the Least Squares variant:
\begin{equation}
    \min_{W_c}\ \frac{1}{S}\sum_{j=1}^S||\mathbf{x}_{i+b,j}W_c-\epsilon_j||^2.
\end{equation}
The Least Squares variant is a convex optimization problem, and has a closed-form solution as follows:
\begin{equation}
    W_{c}=X^+E,
\end{equation}
where $X$ is a matrix created by concatenating all $\mathbf{x}_{i+b,j}$ instances, $E$ is created by concatenating all $\epsilon_j$ instances, and $X^+$ represents the Moore-Penrose inverse of $X$.
When actually caculating for $W_c$, it is time-consuming to calculate $X^+$ as the size of $X$ is big. Therefore, we set the \textit{maximum instance size} of $X$ as $v$. Specifically, we average $\mathbf{x}$'s and $\epsilon$'s for every $\frac{S}{v}$ instances in the initialization set, and calculate $W_c$ on the averaged $\mathbf{x}$'s and $\epsilon$'s. For simplicity, we denote the averaged $\mathbf{x}$'s and $\epsilon$'s as $\mathbb{E}(\mathbf{x}_{i+b})$ and $\mathbb{E}(\epsilon)$. Consequently, the calculation for $W_c$ is formulated as:
\begin{equation}
    W_{c}=V\Sigma^{+}U^T\mathbb{E}(\epsilon),
\end{equation}
where $V\Sigma^{+}U^T$ is the Moore-Penrose inverse of $\mathbb{E}(\mathbf{x}_{i+b})$, calculated via SVD.
\end{proof}

\begin{table}
    \centering
    \resizebox{0.5\columnwidth}{!}{\begin{tabular}{c|ccccccc}
    \toprule
        $v$ & N/A & 1 & 2 & 3 & 4 & 8 & 16 \\
        \midrule
        PPL & 5.914 & 2.890 & 2.871 & 2.835 & 2.850 & 2.822 & 2.837 \\
    \bottomrule
    \end{tabular}}
    \caption{Perplexity results for initializations with different $v$ values on LLaMA-3.1 8B in the medical domain. ``N/A'' represents applying Softmax Unification without linear projections.}
    \label{tab:v_value_for_init}
\end{table}

To find an appropriate $v$ value for initialization, we have conducted preliminary experiments on LLaMA-3.1 8B in the medical domain. Specifically, we have benchmarked the perplexity on the \textit{initialization set} after inserting all linear projections with initialized weights calculated. As shown in \cref{tab:v_value_for_init}, most performance contribution comes from the act of adding initialization, and increasing the number of $v$ affects the PPL values slightly. These experiments guided us to adopt $v=1$ for our experiments due to its simplicity.

\section{Proof for \cref{prop:effective_depth}}
\label{sec:effective_depth_proof}

\begin{proof}
For simplicity, we suppose $\mathbf{x}_{i+1}$ is already normalized by the Pre-Norm. We can re-write \cref{eq:cla} as:
\begin{equation}
\begin{aligned}
    \mathbf{x}'_{i+1}&=\text{softmax}(\frac{\mathbf{x}_{i+1}W_{q,i+1}K^T_{i}}{\sqrt{d_k}})V_iW_{o,i+1}\\&\quad+\mathbf{x}_{i+1} \\
    &=\text{softmax}(\frac{(\mathbf{x}_{i}+\delta_i)W_{q,i+1}K^T_{i}}{\sqrt{d_k}})V_iW_{o,i+1}\\&\quad+\mathbf{x}_{i+1} \\
    &=\text{softmax}(\frac{\mathbf{x}_{i}W_{q,i+1}K^T_{i}}{\sqrt{d_k}}\\&\quad+\frac{\delta_i W_{q,i+1}K^T_{i}}{\sqrt{d_k}})V_iW_{o,i+1}+\mathbf{x}_{i+1}
\end{aligned}
\end{equation}
Let $A_i=\frac{\mathbf{x}_{i}W_{q,i+1}K^T_{i}}{\sqrt{d_k}}$, $\hat{\delta}_i=\frac{\delta_i W_{q,i+1}K^T_{i}}{\sqrt{d_k}}$. By expanding Softmax to the first order:
\begin{equation}
\begin{aligned}
    \mathbf{x}'_{i+1}
    &=\big[\text{softmax}(A_i)+J(A_i)\hat{\delta}_i\big]V_iW_{o,i+1}\\&+\mathbf{x}_{i+1}+o(\hat{\delta}_i^2),
\end{aligned}
\end{equation} 
where $J(A_i)$ is the Jacobian of Softmax at $A_i$. We ignore the $o(\hat{\delta}_i^2)$ remainder. 
\cite{StreamingLLM} has discovered that the pre-trained LLMs generally exhibit an ``attention sink'' feature, in which self-attentions heavily attend to the ``sink token'' at the beginning of a sequence. We adopt such feature to estimate $||J(A_i)||$.
According to the statistical pattern concluded by \cite{StreamingLLM}, attention logits $\mathbf{a}\in \mathbb{R}^{l}$ of a sequence with length $l$ is approximately:
\begin{equation}
    a_i=1\; (0<i<d),\; a_j=-1\; (i\ge d),\; d \ll l
\end{equation}
\cite{StreamingLLM} empirically adopts $d=4$. For a sequence of $l=1024$, $||J(\text{softmax}(\mathbf{a}))||\approx0.03\ll 1$, and decreases as $l$ increases. 
Therefore, we can conclude that the coefficient of the $\hat{\delta}$ term is significantly smaller than that of the Softmax term, making the contribution of the ``depth factor'' negligible.
\end{proof}

\section{Proof for \cref{prop:uniattn_depth}}
\label{sec:uniattn_depth_proof}

\begin{proof}
In the forward pass of UniAttn:
\begin{equation}
\begin{aligned}
    \mathbf{x}_{i+1}'&=s_i(\mathbf{x}_{i+1}W_{v,i+1})W_{o,i+1}+\mathbf{x}_{i+1} \\
    &=s_i(\mathbf{x}_{i}+\delta_i)W_{v,i+1}W_{o,i+1}+\mathbf{x}_{i+1}
\end{aligned}
\end{equation}
Observing that LLM layers primarily change activation magnitudes (see \cref{prop:similarity}), we assume $\mathbf{x}_i$ and $\delta$ sharing similar directions, further leading to $\frac{||s_i\mathbf{x}_{i}W_{v,i+1}W_{o,i+1}||}{||s_i\delta_i W_{v,i+1}W_{o,i+1}||}\approx\frac{||\mathbf{x}_{i}||}{||\delta_i||}$.
Self-attention in UniAttn maintains the relative magnitude of the depth factor, ensuring its influence on the output.
\end{proof}

\section{More Insights on Linear Compensation}
\label{sec:more_linear_comp}
To demonstrate the effectiveness of our linear compensation strategy, we propose the following theorem:
\begin{theorem}
\label{theorem:error_expect}
    $A,B\in\mathbb{R}^{m\times n}$, $X\in \mathbb{R}^{n\times n}$. Suppose that each element from $A$ and $B$ are drawn from a Gaussian distribution such that $a_{ij},b_{ij}\sim N(0,1)$. It satisfies that:
    \begin{equation}
        \mathbb{E}\bigl[\min_X ||AX-B||\bigr]=
            \begin{cases} 
                \sqrt{n(m-n)}, & \text{if } m \ge n,\\
                0,        & \text{if } m < n.
            \end{cases}
    \end{equation}
\end{theorem}
See the following subsections for the proof. In the article, we apply the F15K dataset \cite{rerope2023} to calculate the error $\epsilon$ on LLaMA-2 7B \cite{touvron2023llama} and LLaMA-3.1 8B \cite{LLaMA3} models after inserting the initialized linear transformation $W_c$ according to \cref{theorem:init}. While the sequence length being $m=5120$ and the hidden size of applied models being $n=4096$, the error is 11.09 and 5.76 after linear compensation, correspondingly, which are both \textit{magnitudes lower} than their expectations on random data. Those results demonstrate the effectiveness of our linear compensation strategy.

We give the proof for \cref{theorem:error_expect} in the following subsections.
\subsection{Preliminaries}
To prove \cref{theorem:error_expect}, we adopt the notations in \cref{theorem:error_expect} and prove a series of lemmas first.
\begin{lemma}
\label{lemma:P}
Let $P=AA^+$, $P$ is an orthogonal projection and $\text{Im}(P)=\text{Col}(A)$.
\end{lemma}
\begin{proof}
The Moore-Penrose inverse matrix exhibits some basic features, namely $AA^+A=A$ and $(AA^+)^T=AA^+$. Using that feature we can easily verify that:
\begin{equation}
    P^2=AA^+AA^+=(AA^+A)A^+=AA^+=P
\end{equation}
\begin{equation}
    P^T=(AA^+)^T=AA^+=P
\end{equation}
So $P$ is an orthogonal projection.

Next, we prove $\text{Im}(P)=\text{Col}(A)$. For any vector $\mathbf{v}\in \text{Col}(A)$, there exists a $\mathbf{c}\in \mathbb{R}^n$ that $\mathbf{v}=A\mathbf{c}$. We have:
\begin{equation}
    P\mathbf{v}=AA^+A\mathbf{c}=(AA^+A)\mathbf{c}=A\mathbf{c}=\mathbf{v}
\end{equation}
So $\mathbf{v}\in\text{Im}(P)$. Therefore, $\text{Col}(A)\subseteq\text{Im}(P)$ Since $P=AA^+$, it is obvious that $\text{Im}(P)\subseteq\text{Col}(A)$. Hence, we can conclude that $\text{Im}(P)=\text{Col}(A)$.
\end{proof}

\begin{lemma}
\label{lemma:i_p}
$I-P$ is an orthogonal projection and $\text{Im}(I-P)=\text{Col}(A)^{\perp}$.
\end{lemma}
\begin{proof}
First, we prove that $I-P$ is an orthogonal projection:
\begin{equation}
(I-P)^2=I-2P+P^2=I-2P+P=I-P
\end{equation}
\begin{equation}
(I-P)^T=I-P^T=I-P
\end{equation}
Then, we prove that $\text{Im}(I-P)=\text{Col}(A)^{\perp}$. For any $\mathbf{v}\in \text{Im}(I-P)$, there exists a $\mathbf{c} \in \mathbb{R}^m$ that $\mathbf{v}=(I-P)\mathbf{c}$. Let $\mathbf{d}\in \text{Col}(A)$, then there exists a $\mathbf{w} \in \mathbb{R}^n$ that $\mathbf{d}=A\mathbf{w}$. Computing the inner product of $\mathbf{v}$ and $\mathbf{d}$ yields:
\begin{equation}
\begin{aligned}
    \langle\mathbf{v},\mathbf{d}\rangle&=\langle(I-P)\mathbf{c},A\mathbf{w}\rangle=\langle\mathbf{c},(I-P)A\mathbf{w}\rangle\\&=\langle\mathbf{c},(A-AA^+A)\mathbf{w}\rangle=\langle\mathbf{c},0\rangle=0
\end{aligned}
\end{equation}
This shows that $\mathbf{v}\in\text{Col}(A)^{\perp}$, and in turn $\text{Im}(I-P)\subseteq\text{Col}(A)^{\perp}$.

Conversely, suppose $\mathbf{u}\in\text{Col}(A)^{\perp}$. Recall \cref{lemma:P} that $P$ is an orthogonal projection and $\text{Im}(P)=\text{Col}(A)$, we have $P\mathbf{u}=0$. Therefore,
\begin{equation}
    (I-P)\mathbf{u}=\mathbf{u}-P\mathbf{u}=\mathbf{u}
\end{equation}
This shows that $\mathbf{u}\in\text{Im}(I-P)$, and in turn $\text{Col}(A)^{\perp}\subseteq\text{Im}(I-P)$. Therefore, $\text{Im}(I-P)=\text{Col}(A)^{\perp}$.
\end{proof}

From \cref{lemma:i_p}, we can immediately conclude that:
\begin{corollary}
\label{corollary:rank}
$\text{rank}(I-P)=m-\min (m-n)$
\end{corollary}
\begin{proof}
    $\text{rank}(I-P)=\dim \text{Im}(I-P)=\dim \text{Col}(A)^{\perp}=m-\text{rank}(A)$. Recall \cref{theorem:error_expect} that every element from $A$ is sampled from a Gaussian distribution, so that by probability of 1 we have $\text{rank}(A)=\min(m,n)$, which leads to the conclusion of this corollary.
\end{proof}

Lastly, we give a lemma on orthogonal projections to a Gaussian-sampled vector:
\begin{lemma}
\label{lemma:gaussian}
    If $\mathbf{b}\sim N(\mathbf{0},I_{m\times m})$ is a standard Gaussian in $\mathbb{R}^m$ and $Q$ is an orthogonal projector of rank $k$, then
\begin{equation}
    \mathbb{E}\bigl[||Q\mathbf{b}||\bigr]=\sqrt{k}
\end{equation}
\end{lemma}
\begin{proof}
$\mathbb{E}\bigl[||Q\mathbf{b}||^2\bigr]=\mathbb{E}\bigl[\mathbf{b}^TQ^TQ\mathbf{b}\bigr]=\mathbb{E}\bigl[\mathbf{b}^TQ\mathbf{b}\bigr]=\mathbb{E}\bigl[\text{tr}(\mathbf{b}^TQ\mathbf{b})\bigr]=\mathbb{E}\bigl[\text{tr}(Q\mathbf{b}\mathbf{b}^T)\bigr]=\text{tr}(Q\mathbb{E}\bigl[\mathbf{b}\mathbf{b}^T\bigr])$. As $\mathbf{b}\sim N(\mathbf{0},I_{m\times m})$, $\mathbb{E}\bigl[\mathbf{b}\mathbf{b}^T\bigr]=I$. Therefore, using the trace property of projectors:
\begin{equation}
    \text{tr}(Q\mathbb{E}\bigl[\mathbf{b}\mathbf{b}^T\bigr])=\text{tr}(Q)=k
\end{equation}
In turn, we can conclude that $\mathbb{E}\bigl[||Q\mathbf{b}||\bigr]=\sqrt{k}$.
\end{proof}

\subsection{Proof}
Finally, we give the proof for \cref{theorem:error_expect}.
\begin{proof}
According to \cref{theorem:init}, for given $A$ and $B$, the optimal solution $X^*$ that satisfies $||AX^*-B||=\min_X ||AX-B||$ is given by:
\begin{equation}
    X^*=A^+B,
\end{equation}
where $A^+$ is the Moore-Penrose inverse of $A$. We can re-write the objective as:
\begin{equation}
\begin{aligned}
    \mathbb{E}(||AA^+B-B||^2)&=\mathbb{E}(||(I-AA^+)(-B)||^2)\\&=\mathbb{E}(||(I-AA^+)B||^2)\\&=\sum_{i=1}^n\mathbb{E}(||(I-AA^+)\mathbf{b}_i||^2)
\end{aligned}
\end{equation}
According to \cref{lemma:i_p} and \cref{corollary:rank}, $I-AA^+$ is an orthogonal projection and $\text{rank}(I-AA^+)=m-\min (m-n)$. With \cref{lemma:gaussian}, we know that for any $i$, $\mathbb{E}(||(I-AA^+)\mathbf{b}_i||^2)=m-\min (m-n)$. Therefore, we can conclude that:
\begin{equation}
    \mathbb{E}(||AA^+B-B||^2)=
        \begin{cases} 
            n(m-n), & \text{if } m \ge n,\\
            0,        & \text{if } m < n.
        \end{cases}
\end{equation}
which leads to the conclusion in \cref{theorem:error_expect}.
\end{proof}

\end{document}